\documentclass[a4paper,UKenglish,cleveref, autoref, thm-restate]{lipics-v2021}
\usepackage{booktabs}
\usepackage{pdflscape}
\usepackage{placeins}
\usepackage{svg}
\usepackage{longtable}
\usepackage{todonotes} 
\usepackage{csquotes}
\usepackage{xspace}
\usepackage{tablefootnote}

\title{SAT Encoding of Partial Ordering Models for Graph Coloring Problems } 

\author{Daniel Faber
}{Department of Computer Science, University of Bonn, Germany}{s6dafabe@uni-bonn.de}{}{funded by the Deutsche Forschungsgemeinschaft (DFG, German Research Foundation) – 459420781}

\author{Adalat Jabrayilov
}{Adesso SE, Dortmund, Germany}{adalat.jabrayilov@adesso.de}{https://orcid.org/0000-0002-1098-6358}{}

\author{Petra Mutzel
}{Department of Computer Science, University of Bonn, Germany}{pmutzel@uni-bonn.de}{https://orcid.org/0000-0001-7621-971X}{This work has been funded by the Deutsche Forschungsgemeinschaft (DFG, German Research Foundation) project no.\ 459420781 and under Germany’s Excellence Strategy–EXC-2047/1–390685813, and by the Federal Ministry of Education and Research of Germany and the state of North-Rhine Westphalia as part of the Lamarr-Institute for Machine Learning and Artificial Intelligence, LAMARR22B.}

\authorrunning{D. Faber, A. Jabryilov, and P. Mutzel} 

\Copyright{CC-BY} 



\category{} 
\relatedversion{}
\relatedversiondetails{Full Version}{https://arxiv.org/abs/2403.15961} 
\supplement{}
\supplementdetails[subcategory={Source Code}, cite={}, swhid={}]{Software}{https://github.com/s6dafabe/popsatgcpbcp}

\nolinenumbers 

\EventEditors{John Q. Open and Joan R. Access}
\EventNoEds{2}
\EventLongTitle{42nd Conference on Very Important Topics (CVIT 2016)}
\EventShortTitle{CVIT 2016}
\EventAcronym{CVIT}
\EventYear{2016}
\EventDate{December 24--27, 2016}
\EventLocation{Little Whinging, United Kingdom}
\EventLogo{}
\SeriesVolume{42}
\ArticleNo{23}

\begin{document}

\maketitle

\def\msum{\textstyle\sum}

\def\gcpassilp{\texttt{ASS-I}\xspace}
\def\gcpasssat{\texttt{ASS-S}\xspace}
\def\gcppopilp{\texttt{POP-I}\xspace}
\def\gcppopsat{\texttt{POP-S}\xspace}
\def\gcppophilp{\texttt{POPH-I}\xspace}
\def\gcppophsat{\texttt{POPH-S}\xspace}

\def\bcpassilp{\texttt{ASS-I-B}\xspace}
\def\bcpasssat{\texttt{ASS-S-B}\xspace}
\def\bcppopilp{\texttt{POP-I-B}\xspace}
\def\bcppopsat{\texttt{POP-S-B}\xspace}
\def\bcppophilp{\texttt{POPH-I-B}\xspace}
\def\bcppophsat{\texttt{POPH-S-B}\xspace}



\begin{abstract}
In this paper, we revisit SAT encodings of the partial-ordering based ILP model for the graph coloring problem (GCP) and suggest a generalization for the bandwidth coloring problem (BCP). 
The GCP asks for the minimum number of colors that can be assigned to the vertices of a given graph such that each two adjacent vertices get different colors.
The BCP is a generalization, where each edge has a weight that enforces a minimal \enquote{distance} between the assigned colors, and the goal is to minimize the \enquote{largest} color used.

For the widely studied GCP, we experimentally compare the partial-ordering based SAT encoding to the state-of-the-art approaches on the DIMACS benchmark set. Our evaluation confirms that this SAT encoding is effective for sparse graphs and even outperforms the state-of-the-art on some DIMACS instances.

For the BCP, our theoretical analysis shows that the partial-ordering based SAT and ILP formulations have an asymptotically smaller size than that of the classical assignment-based model.
Our practical evaluation confirms not only a dominance compared to the assignment-based encodings but also to the state-of-the-art approaches on a set of benchmark instances. 
Up to our knowledge, we have solved several open instances of the BCP from the literature for the first time.
\end{abstract}

\section{Introduction}


The graph coloring problem (GCP) asks for assigning a set of positive integers, called \textit{colors}, to the vertices of a graph such that no two adjacent vertices have the same color while minimizing the number of colors used. The problem has numerous applications, e.g.\ in register allocation \cite{RegisterColoring}, scheduling \cite{Leighton1979AGC}, and computing sparse Jacobian matrices \cite{ColoringJacobian}. For this reason, this problem has been the subject of a vast amount of literature (see e.g., \cite{ColoringSurvey1}\cite{ColoringSurvey2} for surveys). 
However, finding an optimal coloring is known to be NP-hard, and compared to other NP-hard problems, like the travelling salesman problem or the knapsack problem, only relatively small instances can be solved to optimality. 
A generalization of the graph coloring problem is the \textit{bandwidth coloring problem} (BCP). In this problem, every edge $\{u,v\}$ in the graph has an additional weight $d(\{u,v\})$, and for a coloring to be valid, the difference of the colors $c(u)$ and $c(v)$ must be at least $d(\{u,v\})$ (i.e.\ $|c(u)-c(v)| \geq d(\{u,v\})$. The goal is to minimize the largest used color.
Note that for uniform edge distances $d(e) = 1$ for all edges $e\in E$, the BCP reduces to the GCP.
The problem has applications in frequency assignment \cite{BW_ILP}, where transmitters close to each other need to be assigned to sufficiently differing frequencies to prevent interference.  

In this paper, we concentrate on exact approaches for solving the above mentioned problems GCP and BCP, in particularly
SAT approaches as well as integer linear programming (ILP) approaches, which are both state-of-the-art for solving coloring problems on graphs (see, e.g., \cite{CLICOLCOM,JabrayilovM22,HebrardK20,Zykov_Coloring}).

SAT approaches are based on encoding the problem as a Boolean Satisfiability problem. A possible encoding consists of introducing color variables $x_{v,i}$, 
where a true assignment of $x_{v,i}$ represents assigning vertex $v$ with color $i$ (e.g., \cite{CLICOLCOM,ASS_SAT}). Other methods are based on \textit{Zykov's tree} induced by Zykov's deletion-contration recurrence (e.g., \cite{HebrardK20,Zykov_Coloring}), in which the models contain variables $s_{u,v}$ that encode if vertices $u$ and $v$ have the same or different colors. Heule, Karahalios and van Hoeve \cite{CLICOLCOM} have introduced the algorithm \textit{CliColCom} in which they alternatingly solve a maximum clique problem and a graph coloring problem using SAT approaches, where the solution from one problem helps finding a solution for the other problem and vice versa.
Most relevant to this work are the SAT encodings suggested by Tamura et al.~\cite{Tamura2009} and Ans{\'{o}}tegui et al.~\cite{Ansotegui2004}, which contains binary variables $y_{v,i}$ for every vertex $v$ and possible color $i$, indicating if color $i$ is smaller than vertex $v$.
Although the experimental evaluation in \cite{Tamura2009} has shown that this encoding has dominated the assignment SAT encoding, these results have caught little attention in the recent literature. 
 
The most natural ILP model is the \textit{assignment-based} ILP model, which directly assigns colors to vertices by introducing binary variables $x_{v,i}$ that (similar to the color variables in SAT) decide if vertex $v$ is assigned to color $i$. A drawback of this formulation is the presence of symmetries in the solution space: For a valid coloring, any permutation of the color labels provides another equivalent solution leading to a significant larger search space. 
Mendez-Diaz and Zabala \cite{VP_Asymm} have suggested additional symmetry-breaking constraints, that completely eliminate this type of symmetry.
Mutzel and Jabrayilov \cite{POP} have proposed ILP formulations, which are based on formulating the coloring problem as a partial-ordering problem (POP). This model suggests ordering the colors and placing the vertices relatively in this order, analogous to the ordering encoding in \cite{Tamura2009} and \cite{Ansotegui2004} for SAT. 
It has been shown that for sparse graphs, the simple assignment and partial-ordering based models show good performance \cite{POP} and that the partial-ordering based models dominate the assignment-based models.
Furthermore, \cite{JabrayilovM22} have shown theoretical advantages of POP over the assignment ILP.
A rather complex but also competitive ILP model is based on the \textit{set covering} formulation \cite{SetCover} (see, e.g.~\cite{heldcolor}), which uses the fact that a coloring describes a partitioning of the vertices into independent sets, and contains a variable $x_s$ for every independent set $s$ in the graph, that decides if $s$ is a set of vertices corresponding to a color class in the coloring. Because there can be an exponential amount of independent sets, the formulation cannot be solved using standard techniques and instead has to be solved using column generation methods.

For the bandwidth coloring problem, there mainly exist numerous heuristic algorithms \cite{BCPHeur}. Two exact approaches are presented in \cite{BW_ILP}. The first approach uses a constraint programming formulation, which contains $|V|$ variables $x(v) \in [1,H]$ for $v \in V$ and $|E|$ constraints $|x(u)-x(v)| \geq d(\{u,v\})$ for every $\{u,v\} \in E$. The second one is based on the assignment-based ILP model, which contains constraints for every edge $\{u,v\}$ and every pair of colors $i,j$ having a smaller difference than $d(\{u,v\})$. A drawback of this model is the high number of constraints, which depends on the size of the edge weights. We are not aware of any other exact approaches for the bandwidth coloring problem in the literature.

\textit{Our contribution.}
Motivated by the recent interest of the ILP community in partial-ordering based ILP models, we revisited 
SAT encodings of the partial-ordering based model for the GCP and generalize them to the BCP. For the GCP, we also strengthen the model using the symmetry-breaking constraints by Mendez-Diaz and Zabala \cite{VP_Asymm} in order to eliminate the inherent symmetries in the solution space. 
Our experimental evaluation for the GCP shows that the partial-ordering based SAT encoding of the POP model outperforms the assignment-based SAT encoding as well as all evaluated ILP formulations from the literature on the DIMACS benchmark set.

Moreover, for the bandwidth coloring problem, we suggest a new modification of the partial-ordering based SAT and ILP models, 
which needs only one constraint per edge and color. Compared to the assignment based model for bandwidth coloring presented in \cite{BW_ILP}, it has an asymptotically smaller number of constraints. This advantage of a more compact formulation size holds true for the SAT as well as the ILP formulations.
Our computational experiments for the bandwidth coloring problem confirm that the new SAT encodings clearly outperform not only the classical assignment-based formulations but also the published state-of-the-art approaches. Our new SAT encodings solve much more instances to provable optimality within one hour of running time than the published approaches and have a significantly lower runtime on a large part of the instances.

\section{State-of-the-art encodings} 
First, we present state-of-the-art encodings (models) that are relevant for our work.
Subsequently, we discuss the state-of-the-art on exact solvers for the GCP and the BCP. 

We use the following notation:
For a graph $G=(V,E)$, we denote its vertex set by $V(G)$ and its edge set by $E(G)$.
Each edge of an undirected graph is a 2-element subset $e=\{u,v\}$ of $V(G)$.
The end vertices $u,v$ of an edge $\{u,v\}$ are called adjacent vertices or neighbors. 
For given positive edge distances $d(u,v)$ for all $u,v\in V(G)$, we denote the average edge distance in $G$ with $\bar{d}$ .
Each valid coloring partitions the vertices into independent sets, where each independent set corresponds to the set of vertices assigned to a specific color.

The formal definitions of the graph coloring variants studied in the paper are as follows. Given an undirected graph $G = (V,E)$, the \emph{graph coloring problem} (GCP) asks for an assignment $c\colon V \to \mathbb{N}$ minimizing $\max_{v\in V}{c(v)}$, such that $c(u) \neq c(v)$ for all $\{u,v\} \in E$. Given an undirected graph $G = (V,E)$ and edge distances $d\colon E \to \mathbb{N}$, the \emph{bandwidth coloring problem} (BCP) asks for an assignment $c\colon V \to \mathbb{N}$ satisfying $|c(u)- c(v)| \geq d(\{u,v\})$ for all $\{u,v\} \in E$, that minimizes $\max_{v\in V}{c(v)}$. 



\subsection{Integer programming formulations}
In the next section, we discuss the ILP models that are relevant for this work.
We use $H$ to denote an arbitrary upper bound on the solutions of GCP and BCP, respectively. 
For example, a trivial upper bound for the GCP is $H=|V|$ and for the BCP is $H=|V|\cdot\max\{d_e \colon e \in E\}$. 

\subsubsection{The assignment models (\gcpassilp) and (\bcpassilp)}
The classical ILP model for graph coloring is based on directly assigning a color $i=1,...,H$ to each of the vertices $v \in V$. For this it introduces binary variables $x_{v,i} \in \{0,1\}$ for each $i=1,...,H$ and $v \in V$, which indicate if color $i$ is assigned to vertex $v$ (in this case $x_{v,i}=1$, otherwise $x_{v,i}=0$).  
To model the objective function, additional binary variables $w_i$ for all colors $i=1,...,H$ are introduced, which indicate if a color $i$ is used. This model is given by:
\begin{subequations}\label{ASS-ILP:main}
\begin{align}
&&\text{min} \quad & \sum_{i = 1}^H w_i              &   &  \notag \\ 
&&\text{s.t.}\quad & \msum_{i = 1}^H x_{v,i} = 1     & \forall & v \in V \label{ASS-ILP:a}  \\
&             && x_{u,i}+x_{v,i} \leq w_i        & \forall & \{u,v\} \in E, i = 1,...,H \label{ASS-ILP:b}\\
&             && w_i \le \msum_{v\in V} x_{vi}   & \forall & i =1,\ldots, H\label{mz:cp1}\\
&             && w_{i} \le w_{i-1}               & \forall & i =2,\ldots, H  \label{mz:cp2} \\
&             && x_{v,i},w_i \in \{0,1\}         & \forall & v \in V, i = 1...H \label{ASS-ILP:c}
\end{align}
\end{subequations}
Equation (\ref{ASS-ILP:a}) ensures that each vertex is colored with exactly one color. Equation (\ref{ASS-ILP:b}) guarantees that adjacent vertices have different colors and that variable $w_i$ is set to $1$ if a vertex is colored with $i$. Finally, the objective minimizes the number of used colors. 
A main drawback of the original model using (\ref{ASS-ILP:a}), (\ref{ASS-ILP:b}), and (\ref{ASS-ILP:c}) only is that there are $\binom H {\chi}$ possibilities to select $\chi$ from $H$ colors. 
This results in many optimal solutions that are symmetric to each other.
In order to overcome this symmetry, Mendez-Diaz and Zabala \cite{VP_Asymm} have suggested to add the constraints (\ref{mz:cp1}) and (\ref{mz:cp2}). 

The model contains additional symmetries that arise due to the arbitrary labeling of the colors: For every valid solution, one can obtain an equivalent solution by swapping the labels of two colors. Mendez-Diaz and Zabala \cite{VP_Asymm} propose additional constraints to break these symmetries:
\begin{subequations}
\begin{align}
x_{v,i} = 0 &&\forall i > v, v \in 1,...,H \label{eq:Asymm_ASS1}\\
x_{v,i} \leq \sum_{u = i-1}^{v-1}x_{u,i-1} &&\forall v\in V\setminus \{1,|V|\},i=2,...,H\label{eq:Asymm_ASS2}
\end{align}
\end{subequations}
%
The assignment model, strengthened with the symmetry-breaking constraints, has the form
\begin{align*}
\gcpassilp:\quad 
\min\Big\{\sum_{i = 1}^H w_i  \colon x,w \text{ satisfy  (\ref{ASS-ILP:a})--(\ref{ASS-ILP:c}), (\ref{eq:Asymm_ASS1})--(\ref{eq:Asymm_ASS2})}  \Big\}.
\end{align*}

\paragraph*{Adaptation of the assignment model to the bandwidth coloring model}
Dias et al.~\cite{BW_ILP} suggested an extension to the assignment model to solve the bandwidth coloring problem. The idea is to modify the edge constraints (\ref{ASS-ILP:b}), such that for every edge $e$ and every pair of colors $i,j$ with $|i-j| < d(e)$, at most one of the two colors can be assigned to the two incident vertices.
The full model presented in the paper is given below:
\begin{subequations}\label{ASS-BW-ILP:main}
\begin{align}
\bcpassilp:\quad
& \text{min}  && z_{max}     &   & \notag \\ 
& \text{s.t.} && \msum_{i = 1}^H x_{v,i} = 1 & \forall & v \in V \label{ASS-BW-ILP:a}  \\
&             && x_{u,i}+x_{v,j} \leq 1 & \forall & e=\{u,v\} \in E, \notag \\
&             &&  & \forall & i,j = 1,...,H \text{ with } |i-j| < d(e) \label{ASS-BW-ILP:b}\\
&             && z_{max} \geq i \cdot x_{v,i} & \forall & v \in V, i = 1,...,H\label{ASS-BW-ILP:c}\\
&             && x_{v,i} \in \{0,1\}, z_{\max} \in \mathbb{R} & \forall & v \in V, i = 1,...,H \label{ASS-BW-ILP:d}
\end{align}
\end{subequations}
To describe the largest used color, the formulation uses a continuous variable $z_{max}$ instead of using $H$ binary variables $w_1,...,w_H$, since in an optimal solution of the BCP the largest assigned color can be greater than the number $\sum_i^H w_i$ of assigned colors. 
Constraints (\ref{ASS-BW-ILP:c}) ensure that if there is a vertex with color $i$, then the largest used color $z_{max}$ is at least $i$, i.e.~there is no used color larger than $z_{max}$.
The correctness of the model has been shown in ~\cite{BW_ILP}, we analyze the size of the model in the following.

\begin{lemma} \label{lemma:ASS-ILP-Size}
\bcpassilp contains $H\cdot|V|+1$ variables and 
$(H+1) \cdot |V| + H\cdot|E|(2\bar{d}-1) - \sum_{e \in E} \big(d(e)^2-d(e)\big)$ constraints. 
\end{lemma}

\begin{proof}
Obviously, the model contains $H\cdot|V|+1$ variables and
 $(H+1) \cdot |V|$ constraints of type (\ref{ASS-BW-ILP:a}) and (\ref{ASS-BW-ILP:c}).
The number of edge constraints in (\ref{ASS-BW-ILP:b}) can be rewritten as 
\begin{align*}
\sum_{e \in E}|\{ (i,j) \in 1,...,H\colon|i-j| < d(e)\}| &= \sum_{e \in E} \Big( H \cdot (2d(e)-1)-(d(e)^2-d(e)) \Big) \nonumber \\
&= 2H\sum_{e \in E}d(e) -H|E|- \sum_{e \in E}\big(d(e)^2-d(e)\big)\nonumber\\
&= H\cdot|E|(2\bar{d}-1) - \sum_{e \in E} \big(d(e)^2-d(e)\big). 
\end{align*}
where $\bar{d}$ is the average edge distance in $G$. 
The first equality can be derived as follows: For every color $i$, the interval of colors $j$ satisfying $|i-j| < d(e)$ is $j \in [i-d(e)+1,i+d(e)-1]$. This interval contains exactly $2d(e)-1$ elements, which for the $H$ colors $i=1,...,H$ leads to $H \cdot (2d(e)-1)$ pairs $(i,j)$ in total. However, we have to subtract the pairs we counted for which $j < 1$ or $j>H$. 
For every $i$ with $i-d(e) < 1$, there exist exactly $d(e)-i$ pairs $(i,j)$ for which $j < 1$: $(i,0),(i,-1),...,(i,i-d(e)+1)$. In total, we have $\sum_{i=1}^{d(e)}(d(e)-i) = (d(e)^2-d(e))/2$ of such pairs. For $j > H$ the situation is symmetrical, leading to a total number of $d(e)^2-d(e)$ of pairs we need to subtract for each edge.
\end{proof}


\subsubsection{The partial-ordering based model (\gcppopilp) for the GCP}
Jabrayilov and Mutzel \cite{POP} have suggested to interpret the coloring problem as a partial-ordering problem (POP). 
An advantage of this model is that it has less inherent symmetries between the colors than the assignment model.
This model considers the colors $1,...,H$ to be linearly ordered. Each vertex is then ordered relative to the colors, i.e., for each vertex its relative position with respect to the colors is determined.
A color is then indirectly assigned to a vertex $v$ if it is neither larger nor smaller than $v$. 
The variables $y_{v,i}$ for all $v \in V$ and $i = 1,...,H$ indicate if color $i$ is smaller than vertex $v$. In case $i$ is smaller than $v$ in the partial order (denoted by $v \succ i$), we have  $y_{v,i} = 1$, otherwise $y_{v,i} = 0$. The color of a vertex is then the smallest color that is not smaller than $v$, i.e., the color $i$ for which $y_{v,i-1}-y_{v,i} = 1$ or in the case  $y_{v,1} = 0$ the color $i = 1$.
The partial-ordering based model has the following form, 
where $q$ is an isolated dummy vertex added to $G$:
\begin{subequations}\label{POP-ILP:main}
\begin{align}
& \text{min}  && 1 + \sum_{i = 1}^H y_{q,i}     &   & \notag \\ 
& \text{s.t.} && y_{v,H} = 0  & \forall & v \in V \label{POP-ILP:a} \\
&             && y_{v,i}-y_{v,i+1} \geq 0 & \forall & v \in V, i = 1,...,H-1 \label{POP-ILP:b} \\
&             && y_{u,1} + y_{v,1} \geq 1 & \forall & \{u,v\} \in E\label{POP-ILP:c} \\
&             && y_{u,i-1} - y_{u,i} + y_{v,i-1} - y_{v,i} \leq 1 & \forall & \{u,v\} \in E, i = 2,...,H\label{POP-ILP:d} \\
&             && y_{q,i}-y_{v,i} \geq 0 & \forall & v \in V, i = 1,...,H-1 \label{POP-ILP:e} \\
&             && y_{v,i} \in \{0,1\} & \forall & v \in V, i = 1,...,H \label{POP-ILP:f} 
\end{align}
\end{subequations}
Constraints (\ref{POP-ILP:a})--(\ref{POP-ILP:c}) ensure that each vertex receives exactly one color from $1,...,H$. 
Every adjacent pair of vertices must receive different colors. This is guaranteed by constraints (\ref{POP-ILP:d}).
Constraints (\ref{POP-ILP:e}) enforce that there is no vertex $v \in V$ with $v \succ q$, i.e., the dummy vertex $q$ has the largest used color.
The objective function minimizes the number of colors $\sum_{i = 1}^H y_{q,i}$ smaller than~$q$ incremented by one for the color assigned to $q$.

The variables of the partial-ordering based model and those of the assignment model are related in the following way:
\begin{subequations}
\begin{align}
x_{v,1} = 1-y_{v,1} &&\forall v \in V  \label{POP-ASS1} \\
x_{v,i} = y_{v,i-1}-y_{v,i} &&\forall v \in V, i = 2,...,H  \label{POP-ASS2}
\end{align}
\end{subequations}
Using these equations, the symmetry-breaking constraints (\ref{eq:Asymm_ASS1})--(\ref{eq:Asymm_ASS2}) can be modified for the partial-ordering based model:
\begin{subequations}
\begin{align}
(\ref{eq:Asymm_ASS1}) \Rightarrow  &&& y_{v,v} = 0           &&\forall v \in 1,...,H  \label{eq:weakAsym_POP}\\
(\ref{eq:Asymm_ASS2}) \Rightarrow  &&& y_{v,i} \leq \sum_{u = i-1}^{v-1} (y_{u,i-1}-y_{u,i}) &&\forall v\in V\setminus\{1,|V|\},i=2,...,H \label{eq:Asymm_POP}
\end{align}
\end{subequations}

The partial-ordering based model, strengthened with the symmetry-breaking constraints, has the form
\begin{align*}
\gcppopilp:\quad 
\min\Big\{1 + \sum_{i = 1}^H y_{q,i} \colon y \text{ satisfy 
(\ref{POP-ILP:a})--(\ref{POP-ILP:f}), 
(\ref{eq:weakAsym_POP})--(\ref{eq:Asymm_POP})} \Big\}.
\end{align*}

Notice that in \cite{POP} the vertex $q$ is chosen from $V$. However, this would cause a conflict with the symmetry-breaking constraints. To avoid the conflict we add $q$ as a new isolated vertex.

\subsubsection{The hybrid partial-ordering based model (\gcppophilp)}
Jabrayilov and Mutzel~\cite{POP} observed that for growing graph density, the constraint matrix of the model (\gcppopilp) contains more nonzero elements than the (\gcpassilp) constraint matrix. This is due to constraints (\ref{POP-ILP:d}), which are responsible for adjacent vertices having different colors and contain four nonzero coefficients instead of the three in the corresponding constraints (\ref{ASS-ILP:b}) in (\gcpassilp). To circumvent this problem, they suggest a hybrid model (\gcppophilp): 
In this model, they include the variables $x_{v,i} \in \{0,1\}$ with the constraints (\ref{POP-ASS1})-(\ref{POP-ASS2}) and substitute the constraints (\ref{POP-ILP:d}) by:
\begin{align}
x_{u,i}+x_{v,i} \leq 1 &&\forall e=\{u,v\} \in E,i=1,...,H  \label{POPH-ILP:a}
\end{align}

The hybrid model, strengthened with the symmetry-breaking constraints, has the form
\begin{align*}
\gcppophilp:\ 
\min\Big\{1 + \sum_{i = 1}^H y_{q,i} \colon x,y \text{ satisfy  (\ref{ASS-ILP:c}),
(\ref{POP-ILP:a})--(\ref{POP-ILP:b}), 
(\ref{POP-ILP:e})--(\ref{POP-ILP:f}), 
(\ref{POP-ASS1}), (\ref{POP-ASS2}), (\ref{POPH-ILP:a}), 
(\ref{eq:weakAsym_POP}), (\ref{eq:Asymm_ASS2})} \Big\}.
\end{align*}

\subsection{SAT encodings (\gcpasssat) and (\bcpasssat)}
Similar to the ILP encoding, the assignment model for graph coloring can also be encoded as a Boolean Satisfiability Problem (SAT). 
Since SAT is a decision problem, we cannot directly optimize the number of used colors. To find the chromatic number of a graph, one therefore encodes the $k$-colorability problem (i.e., the problem of deciding if a given graph can be colored using $k$ colors). 
The assignment constraints (\ref{ASS-ILP:a}) and (\ref{ASS-ILP:b}) are sufficient to model the $k$-colorability.
These constraints can be encoded using the following clauses:
\begin{subequations}
\begin{align}
&&              & \textstyle\bigvee_{i = 1}^k x_{v,i}  & \forall v \in V \label{ASS-SAT:a}  \\
&&              & \neg x_{u,i} \lor \neg x_{v,i} & \forall \{u,v\} \in E,\ i = 1,...,k \label{ASS-SAT:b}\\
&&&              x_{v,i} \in \{True,False\} & \quad \forall v \in V, i = 1,...,k \notag 
\end{align}
\end{subequations}

%
To encode that each vertex is also assigned to at most one color, one possible encoding is the sequential encoding \cite{SeqEncoding}, where the idea is to build a count-and-compare hardware circuit and translate it into conjunctive normal form (CNF). This encoding adds $3k-4$ clauses and $k-1$ auxiliary variables $s_{v,i},i=1,...,k-1$ per vertex $v$:
\begin{subequations}
\begin{align}
& \neg x_{v,i} \lor s_{v,i} & \forall & v \in V,i = 1,...,k-1 \label{eq:seq1}\\
& \neg s_{v,i-1} \lor s_{v,i} & \forall & v \in V,i = 2,...,k-1\label{eq:seq2}\\
& \neg x_{v,i} \lor \neg s_{v,i-1} & \forall & v \in V,i = 2,...,k-1\label{eq:seq3}\\
& \neg x_{v,k} \lor s_{v,k-1} & \forall & v \in V \label{eq:seq4}
\end{align}
\end{subequations}
We remark that enforcing each vertex to have at most one color is not strictly necessary, however, it may improve performance as it eliminates redundant solutions from the search space. In our initial experiments, only enforcing each vertex to have at least one color or using the standard binomial encoding for the at-least-1 constraints showed subpar performance.

Translating the symmetry-breaking constraints (\ref{eq:Asymm_ASS1})-(\ref{eq:Asymm_ASS2}) adds the following clauses:
\begin{subequations}
\begin{align}
\neg x_{v,i} &&\forall i > v, v \in 1,...,k \label{eq:Asymm_ASS1-S}\\
\neg x_{v,i} \lor \bigvee_{u = i-1}^{v-1}x_{u,i-1} &&\forall v\in V\setminus \{1,|V|\},i=2,...,k\label{eq:Asymm_ASS2-S}
\end{align}
\end{subequations}
similar symmetry breaking was also used in \cite{CLICOLCOM,Gelder}.

The SAT encoding of the assignment model, strengthened with these symmetry-breaking constraints, has the following form:

\begin{center}
\gcpasssat:\quad consists of clauses (\ref{ASS-SAT:a}), (\ref{ASS-SAT:b}), 
(\ref{eq:seq1})--(\ref{eq:seq4}),
(\ref{eq:Asymm_ASS1-S}), and (\ref{eq:Asymm_ASS2-S}).
\end{center}



\paragraph*{Adaptation to the bandwidth coloring model}

To extend the previous SAT formulation into a formulation for the bandwidth coloring problem, one can modify the edge clauses (\ref{ASS-SAT:b}) analogous to the ILP model (\bcpassilp). 
The new edge clauses are:
\begin{align}
&\neg x_{u,i} \lor \neg x_{v,j} &\forall & e=\{u,v\} \in E, \forall i,j = 1,...,k \text{ with } |i-j| < d(e)  \label{ASS-SAT-BCP:edge}
\end{align}

The SAT encoding of the assignment model for the BCP has the following form:

\begin{center}
\bcpasssat:\quad consists of clauses (\ref{ASS-SAT:a}), 
(\ref{eq:seq1})--(\ref{eq:seq4}), and
(\ref{ASS-SAT-BCP:edge}).
\end{center}

%
%
%

\section{Formulations based on the partial-ordering approach}

Here, we revisit the SAT encoding suggested by~\cite{Tamura2009} and \cite{Ansotegui2004} for the GCP which also can be seen as the SAT-counterpart to the partial-ordering based ILP model.
We suggest a modification of the symmetry-breaking constraints used in \cite{VP_Asymm} for the partial-ordering based model that can be encoded into SAT in polynomial size and without adding new variables. 
Furthermore, we suggest a new hybrid version inspired by the (\gcppophilp) model.

\subsection{SAT encodings based on partial-ordering: (\gcppopsat) and (\gcppophsat)}

\begin{subequations}
\begin{align}
&             \neg y_{v,k} & \forall &v \in V \label{POP-SAT:a} \\
&          y_{v,i} \lor \neg y_{v,i+1} &\forall& v \in V,i=1,...,k-1\label{POP-SAT:b}  \\
&          y_{u,1} \lor  y_{v,1} &\forall &\{u,v\} \in E \label{POP-SAT:c}\\
&          \neg y_{u,i-1} \lor  y_{u,i} \lor \neg y_{v,i-1} \lor y_{v,i} &\forall &\{u,v\} \in E, i=2,...,k\label{POP-SAT:d}\\
&              y_{v,i} \in \{True,False\} & \forall &v \in V, i = 1,...,k \notag 
\end{align}
\end{subequations}
The clauses (\ref{POP-SAT:a}) guarantee that every vertex is at most as large as color $k$ in the partial order.
Clauses (\ref{POP-SAT:b}) ensure the transitivity of the partial order, i.e., vertex $v$ being larger than color $i$ implies that it is also larger than color $i-1$. Finally, clauses (\ref{POP-SAT:c})-(\ref{POP-SAT:d}) enforce that adjacent vertices must get a different color. In total, the model contains $k\cdot |V|$ variables and $k \cdot(|V|+|E|)$ constraints. However, one can preassign the variables according to clauses (\ref{POP-SAT:a}), reducing the number of variables to $(k-1)\cdot |V|$ and the number of clauses to $k \cdot(|V|+|E|)-|V|$.

Note that the partial-ordering based model directly encodes that each vertex is assigned to \textit{exactly} one color (in contrast, the assignment based model needs additional cardinality constraints to enforce this).
\paragraph*{Adapting symmetry-breaking constraints for the POP-Model}
The translation of the symmetry-breaking constraints (\ref{eq:weakAsym_POP}) into SAT is trivial:
\begin{align}
& \neg y_{v,v}           &&\forall v \in 1,...,k \label{eq:weakAsymm_POP-S}
\end{align}
A drawback of inequality (\ref{eq:Asymm_POP}) is that translating it into a SAT encoding is no longer straightforward.
However, we propose the following simplified inequality, that also eliminates all symmetries arising due to relabeling of the colors:
\begin{align}
y_{v,i} \leq \sum_{u = i-1}^{v-1}y_{u,i-1} &&\forall v\in V\setminus\{1,|V|\},i=2,...,H \label{eq:Asymm_POP2}
\end{align}
\begin{lemma}
Inequalities  (\ref{eq:weakAsym_POP}) and (\ref{eq:Asymm_POP2}) guarantee that for all $i=2,...,k$, the smallest vertex in color class $i$ is larger than the smallest vertex in color class $i-1$. 
\end{lemma}
\begin{proof}
%
In case $i=2$, the claim follows directly from (\ref{eq:weakAsym_POP}).
Assume, for contradiction, $i>2$ is the greatest color such that for the smallest vertex $v$ of $i$ and the smallest vertex $u$ of $i-1$, it holds that $v < u$. 
Since we have $y_{v,i-1} = 1$, according to (\ref{eq:Asymm_POP2}) there must exist a vertex $w \in i-1,...,v-1$, such that $y_{w,i-2} = 1$. Let $w$ be the smallest of such vertices.
From $u>v$ and from the fact that $u$ and $v$ are smallest vertices of colors $i-1$ and $i$ follows that vertex $w$ cannot be colored with $i-1$ or $i$. 
So $w$ must be colored with a color $i^* \ge i+1$. 
The construction of $w$ implies that $w$ is the smallest one of the vertices with colors $i,i+1,...,i^*$.
It follows that $w$ of color $i^*$ is smaller than the smallest vertex of color $i^*-1$. 
This contradicts our assumption that $i$ is the greatest color such that the smallest vertex of $i$ is smaller than the smallest vertex of color $i-1$.
\end{proof}
The advantage of (\ref{eq:Asymm_POP2}) over the naive adaptation is that it can easily be encoded as a set of logical clauses:
\begin{align}
&\neg y_{v,i} \lor \bigvee_{u=i-1}^{v-1} y_{u,i-1} & \forall  &v\in V\setminus\{1,|V|\},i=2,...,k
\label{POP-SAT:asym}
\end{align}

The SAT encoding based on partial-ordering for the GCP has the following form:
\begin{center}
\gcppopsat:\quad consists of clauses  (\ref{POP-SAT:a})--(\ref{POP-SAT:d}), (\ref{eq:weakAsymm_POP-S}) and (\ref{POP-SAT:asym}).
\end{center}

\subsubsection{Hybrid partial-ordering based SAT encoding for the GCP}
One can also encode the hybrid partial-ordering based model as SAT. The clauses corresponding to (\ref{POP-ASS1})-(\ref{POP-ASS2}) are:
\begin{subequations}
\begin{align}
x_{v,1} \lor  y_{v,1}&&\forall v \in V  \label{POP-ASS1-S} \\
\neg x_{v,1} \lor  \neg y_{v,1}&&\forall v \in V  \label{POP-ASS2-S} \\
\neg x_{v,i} \lor y_{v,i-1}&&\forall v \in V, i = 2,...,k  \label{POP-ASS3-S}\\
\neg x_{v,i} \lor \neg y_{v,i} &&\forall v \in V, i = 2,...,k  \label{POP-ASS4-S}\\
x_{v,i} \lor \neg y_{v,i-1} \lor y_{v,i}&&\forall v \in V, i = 2,...,k  \label{POP-ASS5-S}
\end{align}
\end{subequations}
The SAT encoding of the hybrid partial-ordering based model for the GCP has the following form:
\begin{center}
\gcppophsat:\quad consists of clauses  (\ref{POP-SAT:a}),(\ref{POP-SAT:b}),(\ref{ASS-SAT:b}),(\ref{POP-ASS1-S})-(\ref{POP-ASS5-S}), (\ref{eq:weakAsymm_POP-S}) and (\ref{eq:Asymm_ASS2-S}).
\end{center}

\subsection{Partial-ordering based ILP models (\bcppopilp) and (\bcppophilp) for the BCP}
To adapt the partial-ordering based model to the bandwidth coloring problem, one could follow the same approach that was used for the assignment model and use a constraint for every edge $e$ and every pair of colors $i,j$ with $|i-j| < d(e)$. 
However, we suggest an alternative approach, which takes advantage of the fact that the partial-ordering based model orders the vertices with respect to the colors to design a more efficient encoding. The idea of our approach is that the constraint $|c(u)-c(v)|\geq d(e)$ can equivalently be encoded as $c(u) \leq c(v)-d(e) \lor c(u) \geq c(v)+d(e)$, intuitively speaking, the color of $v$ must be at least $d(e)$ greater or less than the color of $u$. This directly leads to the following implication:
\begin{equation*}
c(u) = i \Rightarrow c(v) \leq i-d(e) \lor c(v) \geq i+d(e)
\end{equation*}
By definition, it holds that:
\begin{align*}
c(u) = i &\Leftrightarrow y_{u,i-1} - y_{u,i} = 1\\
c(v) \leq i &\Leftrightarrow y_{v,i} = 0\\
c(v) \geq i &\Leftrightarrow y_{v,i-1} = 1\\
\end{align*}
For the sake of convenience, we define $y_{v,i}:= 1$ for $i < 1$ and $y_{v,i}:= 0$ for $i > H$.
Substituting the terms from the previous implication gives the following constraints:
\begin{align}
y_{u,i-1} - y_{u,i} + y_{v,i-d(e)} - y_{v,i+d(e)-1} \leq 1 && \forall  e=\{u,v\} \in E , i=1,...,H. \label{POPBCP-ILP}
\end{align}
Our new partial-ordering ILP model for the bandwidth coloring problem has the following form:
\begin{align*}
\bcppopilp:\quad 
\min\Big\{1 + \sum_{i = 1}^H y_{q,i} \colon y \text{ satisfy 
(\ref{POP-ILP:a}), (\ref{POP-ILP:b}), (\ref{POPBCP-ILP}), (\ref{POP-ILP:e}), (\ref{POP-ILP:f})} \Big\}.
\end{align*}


\begin{observation}
By Lemma \ref{lemma:ASS-ILP-Size}, the number of constraints in the assignment model 
(\bcpassilp) is
$(H+1) \cdot |V| +H\cdot|E|(2\bar{d}-1) - \sum_{e \in E} \big(d(e)^2-d(e)\big) \overset{H \gg \bar{d}} = \mathcal{O}\big( H\cdot|E|(2\bar{d}-1)\big)$ 
and thus depends on both $\bar d$ and $H$.
In contrast, the number of constraints in the partial-ordering based model (\bcppopilp) is in the order of 
$\mathcal{O}(H\cdot|E|)$ (by straightforward counting), and thus depends only indirectly on the edge weights (via $H$). This gives a size reduction in the order of $\mathcal{O}(\bar{d})$.
%
%
This fact applies analogously to the corresponding SAT encodings.
\end{observation}

\subsubsection{Hybrid partial-ordering ILP model for the BCP}
Analogous to the ILP models for the GCP, one can also formulate a hybrid partial-ordering based model for the BCP having less nonzero terms in the edge constraints than the regular partial-ordering based model. The edge constraints for this model are:
\begin{align}
x_{ui} + y_{v,i-d(e)} - y_{v,i+d(e)-1} \leq 1 && \forall  e=\{u,v\} \in E , i=1,...,H. \label{POPHBCP-ILP}
\end{align}
The model then has the following form:
\begin{align*}
\bcppophilp:\quad 
\min\Big\{1 + \sum_{i = 1}^H y_{q,i} \colon x,y \text{ satisfy  (\ref{ASS-ILP:c}),
(\ref{POP-ILP:a}), (\ref{POP-ILP:b}), 
(\ref{POP-ILP:e})--(\ref{POP-ILP:f}), 
(\ref{POP-ASS1}), (\ref{POP-ASS2}), (\ref{POPHBCP-ILP})} \Big\}.
\end{align*}

\subsection{SAT encodings (\bcppopsat) and (\bcppophsat) based on partial-ordering for the BCP}
The ILP formulations introduced in the previous section can easily be translated into SAT encodings.
For the sake of convenience, we define $y_{v,i}:= True$ for $i < 1$ and $y_{v,i}:= False$ for $i > k$.
The clauses corresponding to constraints (\ref{POPBCP-ILP}) are:
\begin{align}
\neg y_{u,i-1} \lor y_{u,i} \lor \neg y_{v,i-d(e)} \lor y_{v,i+d(e)-1} && \forall  e=\{u,v\} \in E , i=1,...,k. \label{POPBCP-SAT}
\end{align}
which gives us the encoding \bcppopsat as follows:
\begin{center}
\bcppopsat:\quad consists of clauses  (\ref{POP-SAT:a}),(\ref{POP-SAT:b}), and (\ref{POPBCP-SAT}).
\end{center}
\subsubsection{Hybrid partial-ordering SAT encoding for the BCP}
The clauses corresponding to constraints (\ref{POPBCP-ILP}) are:
\begin{align}
\neg x_{u,i} \lor \neg y_{v,i-d(e)} \lor y_{v,i+d(e)-1} && \forall  e=\{u,v\} \in E , i=1,...,k. \label{POPHBCP-SAT}
\end{align}
which gives us the encoding \bcppophsat  as follows:
\begin{center}
\bcppophsat:\quad consists of clauses  (\ref{POP-SAT:a}),(\ref{POP-SAT:b}), (\ref{POP-ASS1-S})--(\ref{POP-ASS5-S}) and (\ref{POPHBCP-SAT}).
\end{center}
\section{Experimental evaluation}
In our computational experiments, we evaluated the effectiveness of the partial-ordering based SAT encodings and compared them with state-of-the art approaches.
In particular, we were interested in a comparison of the partial-ordering based encoding with the assignment based SAT encoding (i.e., the basic SAT encoding) as well as the ILP formulations of the assignment and the partial-ordering based models.
Moreover, we compared the models to state-of-the-art approaches.
The implementation and the data is publically available on \url{https://github.com/s6dafabe/popsatgcpbcp}.
\subsection{Implementation details}
We used the standard preprocessing techniques for graph coloring instances also used in \cite{Gelder,POP}:
\begin{enumerate}[i]
    \item  A vertex $u$ is dominated by a vertex $v$, $v \neq u$, if the neighborhood of $u$ is a subset of the neighborhood of $v$. In this case, the vertex $u$ can be deleted from $G$, the remaining graph can be colored, and at the end, $u$ can get the same color as $v$.
    \item If a vertex $v$ has a degree of less than $L$, where $L$ is a lower bound on the chromatic number, then $v$ can be deleted from $G$ for the calculations. At the end, after the remaining graph has been colored, there is at least one used color left to color $v$ that is not assigned to any of the neighbors of $v$.
    \item Any clique $Q$ represents a lower bound, so one can precolor the vertices in a clique with colors $1,...,|Q|$, eliminating some of the variables. To fix as many variables as possible, one tries to find a clique $Q$ of maximum size.
\end{enumerate}
To reduce the graph as much as possible, we use reductions (i) and (ii) alternatingly until the graph cannot be reduced further. To compute the clique for (iii), we apply the randomised function \texttt{ networkx.maximal\_independent\_set()} on the complement graph of G and choose the best clique out of $300\cdot\frac{|E|}{|V|}$ iterations. 
Another refinement we use is that of all the largest cliques found, we use the one that has the largest cut. 
The motivation behind this is that precoloring a vertex $v$ with a color $i$ also fixes some variables of their neighbors, as it excludes coloring the neighbors of $v$ with $i$. 
As the clique finding method is time consuming for large graphs, we limit the time for finding a clique to $100s$, after which we use the best clique found so far. 

Because SAT is a decision problem, we need to solve a series of $k$-colorability problems to find the chromatic number of a graph. We found that using ascending linear search, i.e.\ starting from a lower bound $L(G)$ and testing satisfiability for $k = L(G),L(G)+1,...,\chi(G)$ until the first satisfiable value for $k$ is found, leads to the best results for the graph coloring problem. For the lower bound $L(G)$, we use the size of the clique found in preprocessing step (3). For the bandwidth coloring problem, we found that descending linear search, i.e.\ testing $k = H(G),H(G)-1,...,\chi(G)$ leads to the best result. To compute an upper bound $H(G)$ for the optimal value for the BCP, we use a simple greedy algorithm:
In every iteration we select the vertex that has not yet been assigned a color and has the highest degree. We then assign the vertex to the smallest color that does not conflict with the colors of any neighbouring vertices that have already been colored. 

Note that we omitted the preprocessing steps (i)-(iii) for the BCP, as they are not applicable for this problem: Because of the distance constraints, swapping the indices between colors may invalidate a coloring, therefore these colorings are not equivalent anymore. Similarly, fixing the colors of vertices in a clique may lead to the optimal solution being excluded. 

%
\subsection{Test setup and benchmark set}
To solve the SAT encodings, we used the solver~\texttt{kissat 3.1.1}\footnote{https://github.com/arminbiere/kissat/releases/tag/rel-3.1.1, accessed 11.03.2024}, which has been successful in the 2022 SAT competition. The preprocessing and the generation of the SAT and ILP formulations were implemented in \texttt{Python 3.10} using the library \texttt{networkx 2.8.5}\footnote{\url{https://networkx.org/documentation/stable/release/release_2.8.5.html}}. 
For solving the ILP models, we used \texttt{Gurobi 10.0.2} single-threadedly. The machine used to evaluate the SAT and ILP formulations features an Intel Xeon Gold 6130@2.1GHz running CentOS Linux and 187 GB of memory (Benchmarks \cite{Benchmark} user time: r500.5=4.87s). For comparison, we compiled the implementation of Heule, Karahalios and Hoeve \cite{CLICOLCOM} and tested it on the same machine. We also compiled the implementation of Held, Cook and Sewell~\cite{heldcolor} with the same Gurobi version. Because their implementation\footnote{https://github.com/heldstephan/exactcolors, accessed 11.03.2024} uses features of Gurobi that are incompatible with the first machine, this method had to be evaluated on a different machine having an AMD EPYC 7543P@2.8GHz and 257GB memory (Benchmarks \cite{Benchmark} user time: r500.5=3.24s). For the graph coloring problem, we performed our experiments on a set of 134 DIMACS graphs \cite{dimacs_graphs} and additionally a set of 9 randomly generated instances by Michael Trick (the \texttt{R}-instances: Note that there exist 18 instances in total, however the instances are duplicated and the duplicates only differ in the node weights, which are irrelevant for standard graph coloring). Furthermore, we compare with the results reported in \cite{CLICOLCOM} of the method presented in \cite{Zykov_Coloring}. For similar reasons as reported in Heule, Karahalios and van Hoeve \cite{CLICOLCOM}, we did not compare to the work in~\cite{CEGAR}, which claims strong results for graph coloring with a method using a relaxed Zykov encoding that is incrementally strengthened. 
The linked source code is currently incorrect, and the authors were unable to reproduce the results.
For the bandwidth coloring problem, we used the \texttt{GEOM} set consisting of 33 graphs generated by Michael Trick \cite{dimacs_graphs}.
We have set a time limit of 1 hour.
\subsection{Experimental results for the graph coloring problem}
\begin{center}
\begin{table}[htb]
\def\Bold#1{\textbf{#1}} 
\small
\setlength{\tabcolsep}{1.5pt} 
\setlength{\extrarowheight}{0.8pt}
\begin{tabular}{lccccccccc}
\toprule
set & \gcpasssat & \gcppopsat & \gcppophsat & \gcpassilp & \gcppopilp & \gcppophilp & EC\cite{heldcolor}  &CLICOL \cite{CLICOLCOM} & CDCL\cite{Zykov_Coloring,CLICOLCOM}~\tablefootnote{The code provided in the repository produced compile errors on our system, so we used the results of the experiments from \cite{CLICOLCOM} which did not contain the \texttt{R}-instances} \\
\midrule
DSJ         & 2 & 2 & 2 & 3 & 2 & 3 & 5 & 4 & 2\\
FullIns    & 14 & 14 & 14 & 12 & 11 & 12 & 5 & 14 & 14\\
Insertions  & 4 & 4 & 4 & 4 & 4 & 4 & 1 & 4 & 3\\
abb313GPI   & 1 & 1 & 1 & 0 & 0 & 0 & 0 & 0& 0\\
anna        & 1 & 1 & 1 & 1 & 1 & 1 & 1 & 1& 1\\
ash        & 3 & 3 & 3 & 3 & 3 & 3 & 0 & 3& 3\\
david       & 1 & 1 & 1 & 1 & 1 & 1 & 1 & 1& 1\\
flat        & 0 & 0 & 0& 0 & 0 & 0 & 2 & 0& 0\\
fpsol2      & 3 & 3 & 3 & 3 & 1 & 3 & 3  &3&3\\
games120   & 1 & 1 & 1 & 1 & 1 & 1 & 1 & 1& 1\\
homer       & 1 & 1 & 1 & 1 & 1 & 1 & 1 & 1& 1\\
huck       & 1 & 1 & 1 & 1 & 1 & 1 & 1& 1 & 1\\
inithx      & 3 & 3 & 3 & 3 & 1& 3 & 3 &3& 3\\
jean        & 1 & 1 & 1 & 1 & 1 & 1 & 1 &1& 1\\
latin\_square  & 0 & 0& 0 & 0& 0 & 0 & 0 &0&0\\
le450       & 8& 8 & 8 & 8 & 7 & 8 & 3&10&8 \\
miles        & 5 & 5 & 5 & 5 & 5 & 5 & 5 & 5&5\\
mug          & 4 & 4 & 4 & 4 & 4 & 4 & 4 & 4&4\\
mulsol       & 5 & 5 & 5 & 5 & 4 & 5 & 5 & 5&5\\
myciel       & 4 & 4 & 4 & 3 & 3 & 3& 2 & 4&4\\
r           & 8 & 7 & 7 & 7 & 5 & 7 & 7 &7&7\\
qg         & 3 & 3 & 3 & 1& 0 & 2 & 0 & 3&3\\
queen        & 6 & \Bold{8} & \Bold{8} & 6 & 5 & 6 &7&5 & 6\\
school1     & 2 & 2 & 2 & 2 & 2 & 2 & 1 & 2&1\\
wap0         & 3 & \Bold{5} & \Bold{5} & 0 & 0 & 1 & 1 &4& 1\\
will199GPI  & 1 & 1 & 1 & 1 & 1 & 1 & 1 & 1&1\\
zeroin  & 3 & 3 & 3 & 3 & 3 & 3 & 3 & 3&3\\
\bottomrule
total & 88 & \Bold{91}&\Bold{91}&79&67&81& 64 &90& 83
\end{tabular}
\caption{Number of solved DIMACS instances on the benchmark set for the GCP}
\label{tab:GCPSolved}
\end{table}
\begin{table}[htb]
\def\Bold#1{\textbf{#1}} 
\small
\setlength{\tabcolsep}{1.5pt} 
\setlength{\extrarowheight}{0.8pt}
\begin{tabular}{lccccccccc}
\toprule
set & \gcpasssat & \gcppopsat & \gcppophsat & \gcpassilp & \gcppopilp & \gcppophilp & EC\cite{heldcolor}  &CLICOL \cite{CLICOLCOM} \\
\midrule
R50             & 3 & 3 & 3 & 3 & 3 & 3 & 3& 3 \\
R75             & 2 & 2 & 2 & 2 & 1 & 2 & 3& 3 \\
R100            & 2 & 2 & 2 & 2 & 1 & 2 & 3& 2 \\
\bottomrule
total & 7 & 7&7&7&5& 7 &9& 8
\end{tabular}
\caption{Number of solved instances on the \texttt{R}-instances}
\label{tab:RSolved}
\end{table}
\end{center}
Table \ref{tab:GCPSolved} shows the number of solved instances for the 134 evaluated DIMACS instances for the SAT encodings of the assignment (\gcpasssat), the partial-ordering (\gcppopsat), and the hybrid partial-ordering (\gcppophsat) based models, the corresponding ILP formulations (\gcpassilp), (\gcppopilp), (\gcppophilp), the method by  Held, Cook, and Sewel~\cite{heldcolor} (\texttt{EC}), the method of Heule, Karahalios and van Hoeve \cite{CLICOLCOM} (\texttt{CLICOL}) and the results of the method of Hebrard and Katsirelos~\cite{Zykov_Coloring} (\texttt{CDCL}) as reported in \cite{CLICOLCOM}. The first column of the table describes the class type and the subsequent columns show the number of solved instances for each model, out of the total number of tested instances.
Table \ref{tab:RSolved} shows the number of solved instances for the 9 \texttt{R}-instances.
Unfortunatly, the code provided ~\cite{Zykov_Coloring} did not compile on our machine, therefore we used the results of the experiments performed in \cite{CLICOLCOM} for the algorithm \texttt{CDCL}. Note that the authors in \cite{CLICOLCOM} did not evaluate on the \texttt{R} instances, which is why they are missing in table \ref{tab:RSolved}.
Note that for the ILP models, the large instances \texttt{DSJC1000.9}, \texttt{latin\_square\_10} and \texttt{qg.order100} resulted in out-of-memory exceptions. Also, the algorithm \texttt{CLICOL} does not seem to be robust for large instances, as it produced runtime errors for the instances \texttt{r1000.5}, \texttt{latin\_square\_10} and \texttt{wap04a}.
The bold items in the table highlight the instance types for which the POP encodings solved more instances than the other methods.

We can see that (\gcppopsat), (\gcppophsat) solved the most DIMACS instances (91/134), followed closely by \texttt{CLICOL} (90/134).
Furthermore, the POP based SAT encodings were able to solve three more instances than (\gcpasssat), and there was only one instance that was solved by (\gcpasssat) but not by the POP based SAT encodings. The POP encodings performed especially well on the \texttt{wap0}-class, as they solved the instances \texttt{wap01a}, \texttt{wap02a} and \texttt{wap08a}, which were only closed recently \cite{CLICOLCOM}.

For the \texttt{R}-instances, we can see that all four SAT methods perform similarly, with \texttt{CLICOL} slightly outperforming the three simple encodings. (\gcpassilp) and (\gcppophilp) also solve the same amount of instances as the three evaluated SAT encodings. \texttt{EC} performs the best, solving all instances, while (\gcppopilp) performs the worst, which is to be expected, as the instance set contains many instances with medium and high density, for which \texttt{EC} typically performs well and (\gcppopilp) typically performs poorly.

It can be observed that the SAT encodings generally outperform  the ILP formulations:
All of the three evaluated SAT encodings solve more instances than all three ILP formulations, even though the underlying models are very similiar. A reason for this observed behaviour is that the LP-relaxations of the assignment and the partial-ordering based models are very weak, which in turn causes the lower bounds derived from the LP-relaxations to be weak. ILP solvers spend a lot of time calculating LP-relaxations during branching to bound the search tree, however, as argued before this technique is not effective for these particular formulations. On the other hand, the clause-learning methods employed by modern SAT-Solvers may work better in this context because they do not rely on the strength of the LP-Relaxation. 

\begin{figure}[tbhp]
  \centering
  \includegraphics[width=1\textwidth]{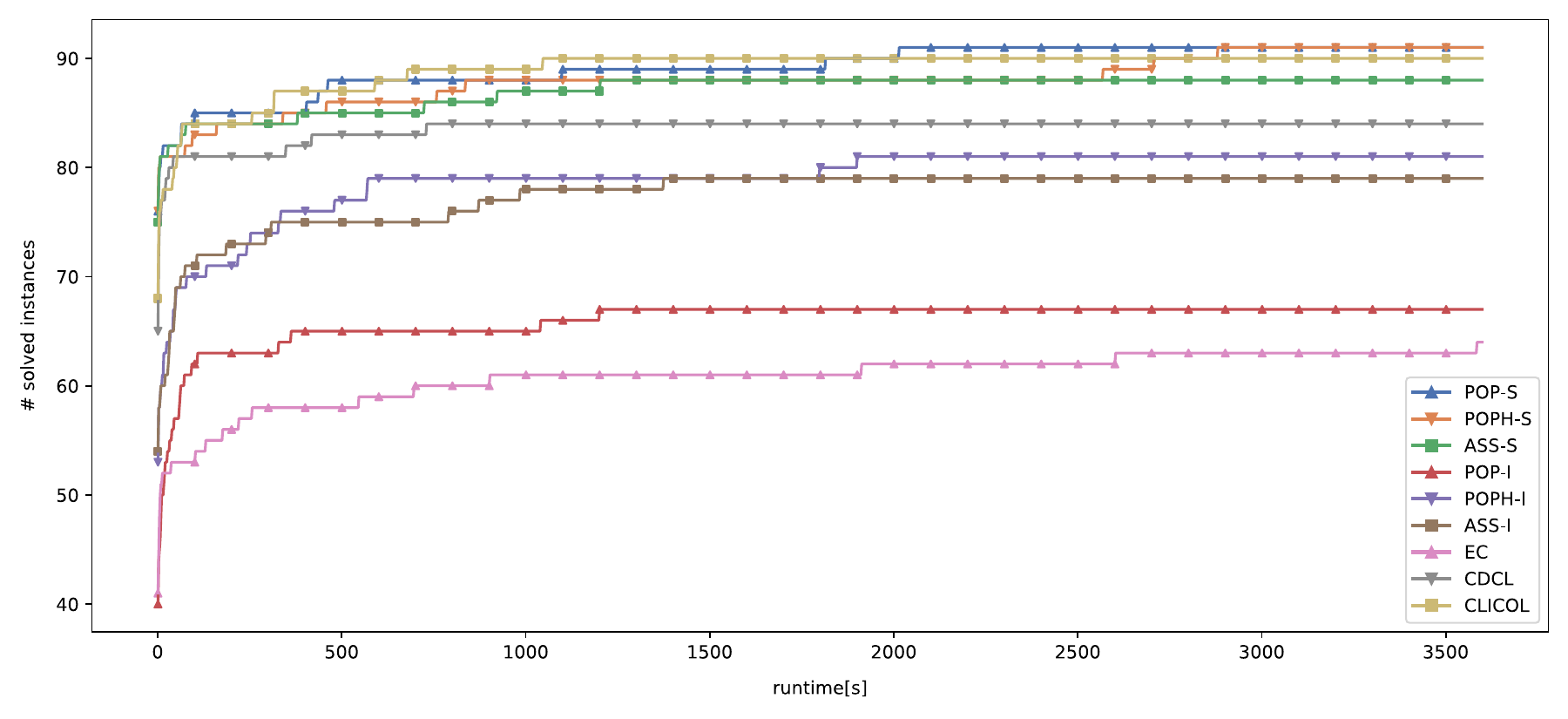}
  \caption{Number of DIMACS instances solved within a given runtime for the GCP}
    \label{fig:Survival_GCP}
\end{figure}

Figure~\ref{fig:Survival_GCP} visualizes
for each model the number of instances, which can be
solved within a time limit of 1, 2,..., 3600 seconds. We omitted the $\texttt{R}$ instances in this figure for better comparability.
We can see that (\gcppopsat) and (\gcppophsat) solve more instances than (\gcpasssat) independent of the considered time limit. Generally, (\gcppopsat), (\gcppophsat) and \texttt{CLICOL} are the best approaches and perform similarly. An interesting observation is that for the ILP formulations, the POP formulation performs far worse than the other two formulations (\gcpassilp) and (\gcppophilp), while for the SAT encodings, (\gcppopsat) and (\gcppophsat) show almost identical performance (with (\gcppopsat) even being slightly better). Jabrayilov and Mutzel \cite{POP} argued that one weakness of the POP ILP formulaton lies in the denser constraint matrix, which is caused by the POP model containing 4 variables in the constraints enforcing differing colors for connected vertices. However, this does not seem to impact the performance of the SAT encoding.

In total (for DIMACS and \texttt{R}-instances combined), (\gcppopsat), (\gcppophsat) and \texttt{CLICOL} solved the most instances (98/143). One interesting thing to note is that although (\gcppopsat)/(\gcppophsat) and \texttt{CLICOL} solved the same number of instances, they solved a different set of instances. For example (\gcppopsat)/(\gcppophsat) is particularly advantageous on the \texttt{queen} instances, while \texttt{CLICOL} shows superior performance on the \texttt{le450} instances. We want to remark that the \texttt{CLICOL} approach uses the assignment-based SAT encoding as a sub-algorithm and combines it with a more sophisticated method of finding an initial clique used for precoloring. An interesting idea could be to use the partial-ordering based SAT encoding in the \texttt{CLICOL} framework to try and combine the advantages of both methods.  

\FloatBarrier

\subsection{Experimental results for the bandwidth coloring problem}


\begin{table} [htbp]
\small
\setlength{\tabcolsep}{3.3pt} 
\begin{tabular}{lrrrrrrrr}
\toprule
set & \#inst. & \bcpasssat & \bcppopsat & \bcppophsat & \bcpassilp & \bcppopilp & \bcppophilp & DFMM\cite{BW_ILP}~\tablefootnote{The authors did not provide the code, so we used the results as reported in \cite{BW_ILP}. We remark that the used time limit in the paper was 24 hour compared to 1 hour in our experiments}\\
\midrule
GEOM{[20-50]}      &  4  &  4  &  4           &  4           &  4  &  4  &  4  &  4  \\
GEOM{[20a-50a]}     &  4  &  4  &  4           &  4           &  4  &  4  &  3  &  4  \\
GEOM{[20b-50b]}     &  4  &  4  &  4           &  4           &  4  &  4  &  4  &  4  \\
GEOM{[60-90]}       &  4  &  4  &  4           &  4           &  4  &  4  &  4  &  4  \\
GEOM{[60a-90a]}     &  4  &  4  &  4           &  4           &  1  &  1  &  0  &  4  \\
GEOM{[60b-90b]}     &  4  &  3  &  \textbf{4}  &  \textbf{4}  &  0  &  0  &  1  &  3  \\
GEOM{[100-120]}     &  3  &  3  &  3           &  3           &  3  &  0  &  0  &  1  \\
GEOM{[100a-120a]}   &  3  &  0  &  \textbf{3}  &  \textbf{3}  &  0  &  0  &  0  &  1  \\
GEOM{[100b-120b]}   &  3  &  0  &  \textbf{2}  &  \textbf{2}  &  0  &  0  &  0  &  1  \\
\bottomrule
\#solved &33 & 26& \textbf{32} & \textbf{32}   &20            &17  &16  & 26 
\end{tabular}
\caption{Number of solved GEOM instances for the BCP}
\label{tab:BCPSolved}
\end{table}
Table \ref{tab:BCPSolved} shows the number of solved instances for the 33 evaluated bandwidth coloring instances for the SAT encodings of the assignment (\bcpasssat), the partial-ordering (\bcppopsat), the hybrid partial-ordering (\bcppophsat) based models, the corresponding ILP formulations (\bcpassilp), (\bcppopilp), (\bcppophilp), and the constraint programming results of the method by 
Dias et al.~\cite{BW_ILP} from the literature. Note that the ILP formulation of the assignment model is equivalent to the model MinGEQ-CDGP-IP used in~\cite{BW_ILP}. 

We can see that (\bcppopsat) and (\bcppophsat) solve the most instances, followed by (\bcpasssat) and the constraint programming approach used in \cite{BW_ILP}. Note that the time limit used in \cite{BW_ILP} is 24 hours compared to just 1 hour in our experiments.
Interestingly, one can observe an opposite trend for the ILP formulations, where the POP formulations are weaker than the assignment formulations. This may be caused by the denser constraint matrices of the POP formulations as argued before, which do not have an impact on the performance of the SAT encodings. 

\begin{figure}[htbp]
  \centering
  \includegraphics[width=1\textwidth]{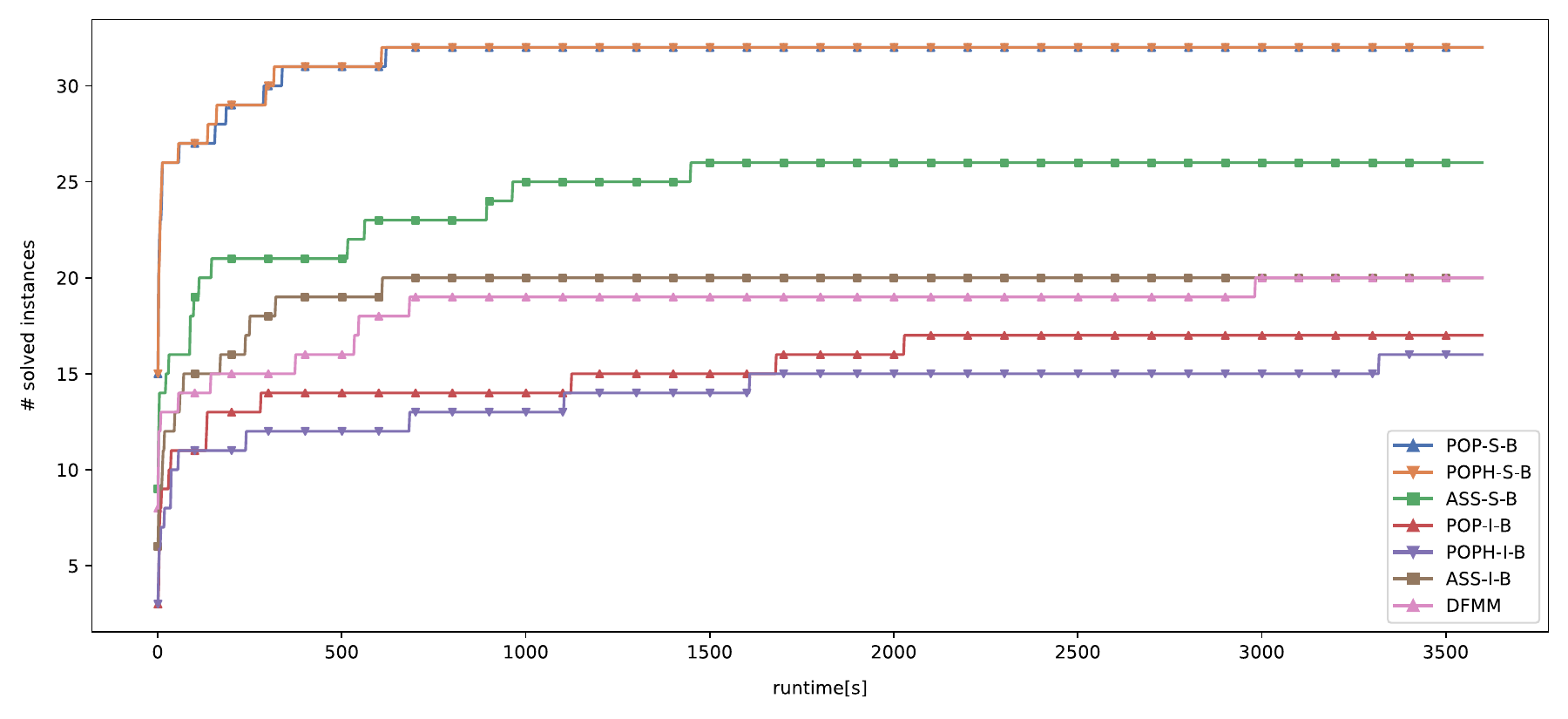}
  \caption{Number of GEOM instances solved within a given runtime for the BCP}
    \label{fig:Survival_BCP}
\end{figure}

Figure~\ref{fig:Survival_BCP} shows the number of solved instances within a time limit of 1, 2,..., 3600 seconds.
One can see that the performance of (\bcppopsat) and (\bcppophsat) is nearly identical and that the two models dominate the other approaches. In particular, the second best approaches (\bcpasssat and constraint programming) solved 26 instances in total, which (\bcppopsat) and (\bcppophsat) both solved in less than 10 seconds; after less than 700 seconds, the POP encodings solved all but one of the 33 GEOM instances to optimality. To our knowledge, this is the first time the instances \texttt{GEOM90b},  \texttt{GEOM100a}, \texttt{GEOM100b}, \texttt{GEOM110a}, \texttt{GEOM110b} and \texttt{GEOM120a} were solved to optimality. 

\section{Conclusion}
In this paper, we have revisited the partial-ordering based ILP and SAT formulations for the graph coloring problem and have suggested new models for the bandwidth coloring problem based on partial-ordering models.

Our computational study on the graph coloring problem suggests that all three SAT encodings perform similar,
with (\gcppopsat) and (\gcppophsat) solving slightly more instances (98/143) than (\gcpasssat) (95/143). This holds true for every timelimit up to 1 hour. Moreover, the SAT encodings solve more instances than the ILP formulations.
Compared to state-of-the-art approaches, the tested SAT based approaches solved more instances than the approach based on the set cover ILP formulations and have shown to be particularly advantageous for sparse graphs. Moreover, the tested SAT based encodings also solved more instances than reported in \cite{Zykov_Coloring} and the same amount of instances as \cite{CLICOLCOM}. Specifically, (\gcppopsat) and (\gcppophsat) have proven to be effective on the \texttt{wap0}- and \texttt{queen}-instances. We also remark that the partial-ordering based encodings are as easy to use as the classical assignment-based encodings. As (\gcppopsat) and (\gcppophsat) have shown superior performance compared to (\gcpasssat), an interesting line of research could therefore be to incorporate the encodings into other SAT-based frameworks, such as the method presented in \cite{CLICOLCOM}.
    
Concerning the bandwidth coloring problem, we have seen that the new POP-based SAT formulations dominate the exact state-of-the-art methods. Compared to the ILP formulations and the constraint programming approach, the SAT-encodings of the POP-based model solve the most instances by far and have a significantly lower runtime on a large part of the instances. This is consistent with the theoretical advantage of the partial-ordering based model, which has an asymptotically smaller formulation size compared to the assignment based model. 

\FloatBarrier

\section{Acknowledgment}
We are grateful for the valuable remarks of the referees, which helped to improve the paper.

\bibliography{literature}

\newpage
\appendix
\section{Detailed results of the nine models for the 143 DIMACS instances for the GCP}\label{sec:Appendix-GCP}
\begin{table}[h!]
\small
\scalebox{0.58}{
  \setlength{\tabcolsep}{1.5pt} 
  \setlength{\extrarowheight}{0.8pt}
  \begin{tabular}{l  rr || ccc | ccc | ccc | ccc | ccc  | ccc | ccc }
               &    &     &    &\gcppopsat   &         &   &\gcppophsat  &         &    &\gcpasssat    &         &    &\gcppopilp   &         &        &\gcppophilp                    &               &        &\gcpassilp                     &         &EC &CLICOL   &CDCL ~\tablefootnote{The code provided in the repository produced compile errors on our system, so we used the results of the experiments from \cite{CLICOLCOM} which did not contain the \texttt{R}-instances}       \\
  Instance     &V   &E    &lb  &ub         &time[s]  &lb  &ub          &time[s]  &lb      &ub     &time[s]  &lb      &ub     &time[s]  &lb      &ub                      &time[s]        &lb      &ub                       &time[s]  &time[s]       &time[s]                      &time[s]       \\
    \hline
1-FullIns\_3  &30  &100  &4  &4  &0.0  &4  &4  &0.0  &4  &4  &0.0  &4  &4  &0.0  &4  &4  &0.0  &4  &4  &0.0  &0.0  &0.2  &0  \\
1-FullIns\_4  &93  &593  &5  &5  &0.1  &5  &5  &0.1  &5  &5  &0.1  &5  &5  &0.1  &5  &5  &0.2  &5  &5  &0.1  &3600.1  &0.1  &0  \\
1-FullIns\_5  &282  &3247  &6  &6  &0.5  &6  &6  &0.6  &6  &6  &0.5  &6  &6  &361.2  &6  &6  &242.9  &6  &6  &29.4  &3600.5  &0.3  &0  \\
1-Insertions\_4  &67  &232  &5  &5  &1.4  &5  &5  &1.2  &5  &5  &1.0  &5  &5  &71.7  &5  &5  &480.4  &5  &5  &106.9  &3600.1  &1.9  &3600.0  \\
1-Insertions\_5  &202  &1227  &4  &6  &3600.0  &4  &6  &3600.0  &4  &6  &3600.0  &4  &6  &3600.0  &4  &6  &3600.0  &4  &6  &3600.0  &3602.6  &3600.0  &3600.0  \\
1-Insertions\_6  &607  &6337  &4  &7  &3600.0  &4  &7  &3600.0  &4  &7  &3600.0  &4  &7  &3600.1  &4  &7  &3600.0  &4  &7  &3600.0  &3600.0  &3600.0  &3600.0  \\
2-FullIns\_3  &52  &201  &5  &5  &0.0  &5  &5  &0.0  &5  &5  &0.0  &5  &5  &0.0  &5  &5  &0.0  &5  &5  &0.0  &0.0  &0.5  &0  \\
2-FullIns\_4  &212  &1621  &6  &6  &0.0  &6  &6  &0.0  &6  &6  &0.1  &6  &6  &1.8  &6  &6  &0.4  &6  &6  &0.2  &3600.1  &0.1  &0  \\
2-FullIns\_5  &852  &12201  &7  &7  &0.7  &7  &7  &0.6  &7  &7  &0.8  &6  &7  &3600.0  &7  &7  &1899.5  &7  &7  &983.6  &3602.5  &3.2  &20  \\
2-Insertions\_3  &37  &72  &4  &4  &0.0  &4  &4  &0.0  &4  &4  &0.0  &4  &4  &0.1  &4  &4  &0.1  &4  &4  &0.1  &220.2  &0.1  &0  \\
2-Insertions\_4  &149  &541  &3  &5  &3600.0  &3  &5  &3600.0  &3  &5  &3600.0  &4  &5  &3600.0  &4  &5  &3600.0  &4  &5  &3600.0  &3601.5  &3600.0  &3600.0  \\
2-Insertions\_5  &597  &3936  &3  &6  &3600.0  &3  &6  &3600.0  &3  &6  &3600.0  &3  &6  &3600.1  &3  &6  &3600.0  &3  &6  &3600.0  &3600.0  &3600.0  &3600.0  \\
3-FullIns\_3  &80  &346  &6  &6  &0.1  &6  &6  &0.1  &6  &6  &0.1  &6  &6  &0.0  &6  &6  &0.0  &6  &6  &0.0  &0.0  &0.5  &0  \\
3-FullIns\_4  &405  &3524  &7  &7  &0.1  &7  &7  &0.1  &7  &7  &0.1  &7  &7  &7.6  &7  &7  &1.4  &7  &7  &0.4  &3600.2  &0.4  &0  \\
3-FullIns\_5  &2030  &33751  &8  &8  &4.1  &8  &8  &5.9  &8  &8  &5.8  &7  &8  &3600.0  &7  &8  &3600.0  &7  &8  &3600.0  &3602.9  &43.0  &0  \\
3-Insertions\_3  &56  &110  &4  &4  &0.0  &4  &4  &0.0  &4  &4  &0.0  &4  &4  &0.8  &4  &4  &1.1  &4  &4  &1.4  &3600.1  &0.2  &1  \\
3-Insertions\_4  &281  &1046  &3  &5  &3600.0  &3  &5  &3600.0  &3  &5  &3600.0  &3  &5  &3600.0  &4  &5  &3600.0  &4  &5  &3600.0  &3613.8  &3600.0  &3600.0  \\
3-Insertions\_5  &1406  &9695  &3  &6  &3600.0  &3  &6  &3600.0  &3  &6  &3600.0  &3  &6  &3600.1  &3  &6  &3600.0  &3  &6  &3600.0  &3600.0  &3600.0  &3600.0  \\
4-FullIns\_3  &114  &541  &7  &7  &0.1  &7  &7  &0.1  &7  &7  &0.0  &7  &7  &0.0  &7  &7  &0.0  &7  &7  &0.0  &0.0  &0.2  &0  \\
4-FullIns\_4  &690  &6650  &8  &8  &0.4  &8  &8  &0.5  &8  &8  &0.6  &8  &8  &60.1  &8  &8  &0.6  &8  &8  &0.5  &3600.7  &1.2  &0  \\
4-FullIns\_5  &4146  &77305  &9  &9  &63.8  &9  &9  &93.3  &9  &9  &75.6  &8  &9  &3600.0  &8  &9  &3600.0  &8  &9  &3600.0  &3629.8  &678.1  &730  \\
4-Insertions\_3  &79  &156  &4  &4  &0.1  &4  &4  &0.1  &4  &4  &0.2  &4  &4  &17.5  &4  &4  &15.7  &4  &4  &48.3  &3600.2  &0.2  &417  \\
4-Insertions\_4  &475  &1795  &3  &5  &3600.0  &3  &5  &3600.0  &3  &5  &3600.0  &3  &5  &3600.0  &3  &5  &3600.0  &3  &5  &3600.0  &3600.0  &3600.0  &3600.0  \\
5-FullIns\_3  &154  &792  &8  &8  &0.2  &8  &8  &0.1  &8  &8  &0.2  &8  &8  &0.0  &8  &8  &0.0  &8  &8  &0.0  &0.0  &0.4  &0  \\
5-FullIns\_4  &1085  &11395  &9  &9  &1.1  &9  &9  &1.4  &9  &9  &3.1  &9  &9  &31.8  &9  &9  &0.9  &9  &9  &0.9  &3600.3  &7.0  &0  \\
abb313GPIA  &1555  &53356  &9  &9  &0.5  &9  &9  &0.2  &9  &9  &0.2  &8  &9  &3600.1  &8  &10  &3600.0  &8  &9  &3600.0  &3600.0  &3600.0  &3600.0  \\
anna  &138  &493  &11  &11  &0.0  &11  &11  &0.0  &11  &11  &0.0  &11  &11  &0.0  &11  &11  &0.0  &11  &11  &0.0  &0.0  &0.4  &0  \\
ash331GPIA  &662  &4181  &4  &4  &0.1  &4  &4  &0.1  &4  &4  &0.1  &4  &4  &6.3  &4  &4  &16.0  &4  &4  &28.6  &3611.7  &0.2  &0  \\
ash608GPIA  &1216  &7844  &4  &4  &0.1  &4  &4  &0.1  &4  &4  &0.0  &4  &4  &10.7  &4  &4  &34.1  &4  &4  &74.3  &3600.0  &0.5  &4  \\
ash958GPIA  &1916  &12506  &4  &4  &0.1  &4  &4  &0.1  &4  &4  &0.1  &4  &4  &57.5  &4  &4  &217.9  &4  &4  &872.6  &3600.0  &1.4  &29  \\
david  &87  &406  &11  &11  &0.1  &11  &11  &0.0  &11  &11  &0.0  &11  &11  &0.0  &11  &11  &0.0  &11  &11  &0.0  &0.0  &0.0  &0  \\
DSJC1000.1  &1000  &49629  &6  &27  &3600.0  &6  &27  &3600.0  &6  &27  &3600.0  &6  &0  &3600.0  &6  &0  &3600.0  &6  &0  &3600.0  &3600.0  &3600.0  &3600.0  \\
DSJC1000.5  &1000  &249826  &16  &115  &3600.0  &16  &115  &3600.0  &16  &115  &3600.0  &14  &0  &3603.0  &14  &0  &3600.1  &14  &0  &3600.8  &3600.0  &3600.0  &3600.0  \\
DSJC1000.9  &1000  &449449  &60  &299  &3600.0  &59  &299  &3600.0  &59  &299  &3600.0  &-  &-  &-  &-  &-  &-  &-  &-  &-  &3600.0  &3600.0  &3600.0  \\
DSJC125.1  &125  &736  &5  &5  &0.0  &5  &5  &0.1  &5  &5  &0.1  &5  &5  &1.4  &5  &5  &2.5  &5  &5  &2.1  &102.7  &0.1  &0  \\
DSJC125.5  &125  &3891  &13  &22  &3600.0  &13  &22  &3600.0  &13  &22  &3600.0  &11  &20  &3600.0  &13  &0  &3600.0  &13  &19  &3600.0  &3600.5  &3600.0  &3600.0  \\
DSJC125.9  &125  &6961  &38  &51  &3600.0  &38  &51  &3600.0  &38  &51  &3600.0  &35  &47  &3600.0  &41  &44  &3600.0  &42  &44  &3600.0  &6.2  &3600.0  &3600.0  \\
DSJC250.1  &250  &3218  &5  &10  &3600.0  &6  &10  &3600.0  &6  &10  &3600.0  &5  &10  &3600.0  &5  &10  &3600.0  &5  &10  &3600.0  &3600.0  &3600.0  &3600.0  \\
DSJC250.5  &250  &15668  &14  &37  &3600.0  &14  &37  &3600.0  &14  &37  &3600.0  &12  &0  &3600.0  &15  &0  &3600.0  &14  &0  &3600.0  &3606.8  &3600.0  &3600.0  \\
DSJC250.9  &250  &27897  &44  &92  &3600.0  &44  &92  &3600.0  &44  &92  &3600.0  &41  &0  &3600.1  &47  &0  &3600.0  &43  &0  &3600.0  &3604.0  &3600.0  &3600.0  \\
DSJC500.1  &500  &12458  &6  &16  &3600.0  &6  &16  &3600.0  &6  &16  &3600.0  &6  &0  &3600.0  &6  &0  &3600.0  &6  &0  &3600.0  &3600.0  &3600.0  &3600.0  \\
DSJC500.5  &500  &62624  &15  &65  &3600.0  &15  &65  &3600.0  &15  &65  &3600.0  &13  &0  &3600.2  &14  &0  &3600.0  &13  &0  &3600.0  &3600.0  &3600.0  &3600.0  \\
DSJC500.9  &500  &112437  &53  &170  &3600.0  &53  &170  &3600.0  &53  &170  &3600.0  &48  &0  &3600.1  &47  &0  &3600.0  &47  &0  &3600.0  &3637.7  &3600.0  &3600.0  \\
DSJR500.1  &500  &3555  &12  &12  &0.0  &12  &12  &0.0  &12  &12  &0.0  &12  &12  &1.2  &12  &12  &0.3  &12  &12  &0.2  &255.1  &0.9  &0  \\
DSJR500.1c  &500  &121275  &81  &89  &3600.0  &80  &89  &3600.0  &80  &89  &3600.0  &74  &0  &3600.1  &81  &0  &3600.0  &75  &89  &3600.3  &695.1  &52.8  &3600.0  \\
DSJR500.5  &500  &58862  &122  &131  &3600.0  &122  &131  &3600.0  &122  &131  &3600.0  &115  &0  &3600.1  &122  &122  &334.5  &122  &122  &1373.1  &2602.2  &65.8  &3600.0  \\
flat1000\_50\_0  &1000  &245000  &16  &114  &3600.0  &16  &114  &3600.0  &16  &114  &3600.0  &14  &0  &3601.3  &14  &0  &3600.0  &14  &0  &3600.7  &3600.0  &3600.0  &3600.0  \\
flat1000\_60\_0  &1000  &245830  &16  &114  &3600.0  &16  &114  &3600.0  &16  &114  &3600.0  &13  &0  &3601.2  &13  &0  &3600.0  &13  &0  &3600.0  &3600.0  &3600.0  &3600.0  \\
flat1000\_76\_0  &1000  &246708  &16  &115  &3600.0  &16  &115  &3600.0  &16  &115  &3600.0  &13  &0  &3601.1  &13  &0  &3600.3  &13  &0  &3600.8  &3600.0  &3600.0  &3600.0  \\
flat300\_20\_0  &300  &21375  &14  &42  &3600.0  &14  &42  &3600.0  &14  &42  &3600.0  &11  &0  &3600.2  &13  &0  &3600.0  &13  &0  &3600.0  &545.8  &3600.0  &3600.0  \\
flat300\_26\_0  &300  &21633  &14  &41  &3600.0  &14  &41  &3600.0  &14  &41  &3600.0  &12  &0  &3600.1  &13  &0  &3600.0  &13  &0  &3600.0  &901.4  &3600.0  &3600.0  \\
flat300\_28\_0  &300  &21695  &14  &42  &3600.0  &14  &42  &3600.0  &14  &42  &3600.0  &12  &0  &3600.1  &14  &0  &3600.0  &14  &0  &3600.0  &3613.3  &3600.0  &3600.0  \\
fpsol2.i.1  &269  &11654  &65  &65  &0.1  &65  &65  &0.1  &65  &65  &0.1  &65  &65  &1.2  &65  &65  &0.3  &65  &65  &0.3  &0.4  &0.6  &0  \\
fpsol2.i.2  &363  &8691  &30  &30  &0.1  &30  &30  &0.2  &30  &30  &0.0  &28  &30  &3600.0  &30  &30  &0.2  &30  &30  &0.4  &0.4  &0.3  &0  \\
fpsol2.i.3  &363  &8688  &30  &30  &0.0  &30  &30  &0.0  &30  &30  &0.0  &29  &30  &3600.0  &30  &30  &0.2  &30  &30  &0.4  &0.4  &0.1  &0  \\
games120  &120  &638  &9  &9  &0.0  &9  &9  &0.0  &9  &9  &0.0  &9  &9  &0.0  &9  &9  &0.0  &9  &9  &0.0  &0.0  &0.1  &0  \\ 
homer  &556  &1629  &13  &13  &0.0  &13  &13  &0.0  &13  &13  &0.0  &13  &13  &61.2  &13  &13  &0.1  &13  &13  &0.1  &0.1  &0.30  &0  \\
huck  &74  &301  &11  &11  &0.0  &11  &11  &0.0  &11  &11  &0.0  &11  &11  &0.2  &11  &11  &0.0  &11  &11  &0.0  &0.0  &0.1  &0  \\
inithx.i.1  &519  &18707  &54  &54  &0.1  &54  &54  &0.0  &54  &54  &0.1  &54  &54  &0.8  &54  &54  &0.3  &54  &54  &0.3  &1.3  &0.3  &0  \\
inithx.i.2  &558  &13979  &31  &31  &0.1  &31  &31  &0.0  &31  &31  &0.1  &29  &31  &3600.0  &31  &31  &0.2  &31  &31  &0.4  &0.2  &0.4  &0  \\
inithx.i.3  &559  &13969  &31  &31  &0.1  &31  &31  &0.2  &31  &31  &0.1  &29  &31  &3600.0  &31  &31  &0.2  &31  &31  &0.3  &0.3  &0.2  &0  \\
jean  &77  &254  &10  &10  &0.0  &10  &10  &0.1  &10  &10  &0.0  &10  &10  &0.0  &10  &10  &0.0  &10  &10  &0.0  &0.0  &0.1  &0  \\
latin\_square\_10  &900  &307350  &90  &132  &3600.0  &90  &132  &3600.0  &90  &132  &3600.0  &-  &-  &-  &-  &-  &-  &-  &-  &-  &3638.4  &- &3600.0  \\
le450\_15a  &450  &8168  &15  &15  &0.8  &15  &15  &0.3  &15  &15  &0.4  &15  &16  &3600.1  &15  &15  &328.1  &15  &15  &307.6  &3633.7  &1.1  &4  \\
le450\_15b  &450  &8169  &15  &15  &0.7  &15  &15  &0.2  &15  &15  &0.2  &15  &15  &1039.7  &15  &15  &567.0  &15  &15  &293.2  &3600.0  &0.3  &1  \\
le450\_15c  &450  &16680  &15  &23  &3600.0  &15  &23  &3600.0  &15  &23  &3600.0  &15  &23  &3600.0  &15  &0  &3600.0  &15  &0  &3600.7  &3600.0  &63.6  &3600.0  \\
le450\_15d  &450  &16750  &15  &24  &3600.0  &15  &24  &3600.0  &15  &24  &3600.0  &15  &0  &3600.0  &15  &0  &3600.0  &15  &0  &3600.4  &3600.0  &53.8  &3600.0  \\
le450\_25a  &450  &8260  &25  &25  &0.1  &25  &25  &0.1  &25  &25  &0.1  &25  &25  &26.8  &25  &25  &1.3  &25  &25  &0.6  &1.8  &0.2  &0  \\
le450\_25b  &450  &8263  &25  &25  &0.0  &25  &25  &0.0  &25  &25  &0.0  &25  &25  &10.9  &25  &25  &0.8  &25  &25  &0.6  &1.9  &0.2  &0  \\

\hline
\end{tabular}
}
\vspace*{5pt}
\caption{Results of the nine models for the 134 DIMACS instances and the 9 \texttt{R}-instances for the GCP}
\label{tableVcpDimacs}
\end{table}

\begin{table}[h!]
\small
\scalebox{0.57}{
  \setlength{\tabcolsep}{1.5pt} 
  \setlength{\extrarowheight}{0.8pt}
  \begin{tabular}{l  rr || ccc | ccc | ccc | ccc | ccc  | ccc | ccc }
               &    &     &    &\gcppopsat   &         &   &\gcppophsat  &         &    &\gcpasssat    &         &    &\gcppopilp   &         &        &\gcppophilp                    &               &        &\gcpassilp                     &         &EC &CLICOL   &CDCL       \\
  Instance     &V   &E    &lb  &ub         &time[s]  &lb  &ub          &time[s]  &lb      &ub     &time[s]  &lb      &ub     &time[s]  &lb      &ub                      &time[s]        &lb      &ub                       &time[s]  &time[s]       &time[s]                      &time[s]       \\
    \hline
le450\_25c  &450  &17343  &25  &29  &3600.0  &25  &29  &3600.0  &25  &29  &3600.0  &25  &29  &3600.0  &25  &28  &3600.0  &25  &29  &3600.1  &3690.4  &3600.0  &3600.0  \\
le450\_25d  &450  &17425  &25  &29  &3600.0  &25  &29  &3600.0  &25  &29  &3600.0  &25  &29  &3600.1  &25  &28  &3600.0  &25  &29  &3600.0  &3637.8  &3600.0  &3600.0  \\
le450\_5a  &450  &5714  &5  &5  &0.1  &5  &5  &0.0  &5  &5  &0.0  &5  &5  &108.5  &5  &5  &23.2  &5  &5  &62.9  &3600.0  &0.1  &41  \\
le450\_5b  &450  &5734  &5  &5  &0.0  &5  &5  &0.0  &5  &5  &0.0  &5  &5  &328.8  &5  &5  &41.2  &5  &5  &47.6  &3600.0  &0.1  &9  \\
le450\_5c  &450  &9803  &5  &5  &0.0  &5  &5  &0.1  &5  &5  &0.0  &5  &5  &92.4  &5  &5  &47.5  &5  &5  &19.3  &3600.0  &0.1  &2  \\
le450\_5d  &450  &9757  &5  &5  &0.1  &5  &5  &0.1  &5  &5  &0.0  &5  &5  &19.8  &5  &5  &49.8  &5  &5  &46.8  &3584.7  &0.1  &2  \\
miles1000  &128  &3216  &42  &42  &0.1  &42  &42  &0.1  &42  &42  &0.1  &42  &42  &0.5  &42  &42  &0.1  &42  &42  &0.1  &0.1  &0.6  &0  \\
miles1500  &128  &5198  &73  &73  &0.1  &73  &73  &0.2  &73  &73  &0.2  &73  &73  &0.5  &73  &73  &0.3  &73  &73  &0.3  &0.1  &0.1  &0  \\
miles250  &125  &387  &8  &8  &0.1  &8  &8  &0.0  &8  &8  &0.0  &8  &8  &0.0  &8  &8  &0.0  &8  &8  &0.0  &0.0  &0.1  &0  \\
miles500  &128  &1170  &20  &20  &0.1  &20  &20  &0.0  &20  &20  &0.1  &20  &20  &0.0  &20  &20  &0.0  &20  &20  &0.0  &0.0  &0.1  &0  \\
miles750  &128  &2113  &31  &31  &0.0  &31  &31  &0.0  &31  &31  &0.0  &31  &31  &0.1  &31  &31  &0.1  &31  &31  &0.1  &0.0  &0.1  &0  \\
mug100\_1  &100  &166  &4  &4  &0.1  &4  &4  &0.0  &4  &4  &0.0  &4  &4  &0.2  &4  &4  &0.1  &4  &4  &0.2  &0.6  &0.1  &0  \\
mug100\_25  &100  &166  &4  &4  &0.0  &4  &4  &0.0  &4  &4  &0.0  &4  &4  &0.2  &4  &4  &0.4  &4  &4  &0.2  &0.5  &0.1  &0  \\
mug88\_1  &88  &146  &4  &4  &0.0  &4  &4  &0.0  &4  &4  &0.0  &4  &4  &0.1  &4  &4  &0.1  &4  &4  &0.2  &0.3  &0.1  &0  \\
mug88\_25  &88  &146  &4  &4  &0.0  &4  &4  &0.0  &4  &4  &0.0  &4  &4  &0.2  &4  &4  &0.1  &4  &4  &0.2  &0.3  &0.1  &0  \\
mulsol.i.1  &138  &3925  &49  &49  &0.0  &49  &49  &0.1  &49  &49  &0.1  &49  &49  &0.4  &49  &49  &0.1  &49  &49  &0.1  &0.1  &0.1  &0  \\
mulsol.i.2  &173  &3885  &31  &31  &0.0  &31  &31  &0.0  &31  &31  &0.0  &31  &31  &0.3  &31  &31  &0.1  &31  &31  &0.1  &0.0  &0.1  &0  \\
mulsol.i.3  &174  &3916  &31  &31  &0.0  &31  &31  &0.0  &31  &31  &0.0  &31  &31  &0.2  &31  &31  &0.1  &31  &31  &0.1  &0.0  &0.2  &0  \\
mulsol.i.4  &175  &3946  &31  &31  &0.0  &31  &31  &0.0  &31  &31  &0.0  &31  &31  &0.4  &31  &31  &0.1  &31  &31  &0.1  &0.0  &0.1  &0  \\
mulsol.i.5  &176  &3973  &31  &31  &0.0  &31  &31  &0.0  &31  &31  &0.0  &30  &31  &3600.0  &31  &31  &0.1  &31  &31  &0.1  &0.0  &0.1  &0  \\
myciel3  &11  &20  &4  &4  &0.0  &4  &4  &0.1  &4  &4  &0.0  &4  &4  &0.0  &4  &4  &0.0  &4  &4  &0.0  &0.0  &0.1  &0  \\
myciel4  &23  &71  &5  &5  &0.0  &5  &5  &0.0  &5  &5  &0.0  &5  &5  &0.1  &5  &5  &0.1  &5  &5  &0.1  &4.8  &0.1  &0  \\
myciel5  &47  &236  &6  &6  &0.2  &6  &6  &0.2  &6  &6  &0.2  &6  &6  &37.5  &6  &6  &42.0  &6  &6  &44.3  &3600.0  &0.5  &0  \\
myciel6  &95  &755  &7  &7  &99.4  &7  &7  &74.3  &7  &7  &63.6  &6  &7  &3600.0  &6  &7  &3600.0  &5  &7  &3600.0  &3600.1  &1045.3  &0  \\
myciel7  &191  &2360  &6  &8  &3600.0  &6  &8  &3600.0  &6  &8  &3600.0  &5  &8  &3600.0  &5  &8  &3600.0  &5  &8  &3600.0  &3600.7  &3600.0  &0  \\
qg.order100  &10000  &990000  &100  &116  &3600.0  &100  &116  &3600.0  &100  &116  &3600.0  &-  &-  &-  &-  &-  &-  &-  &-  &-  &3600.0  &3600.0  &3600.0  \\
qg.order30  &900  &26100  &30  &30  &0.8  &30  &30  &0.4  &30  &30  &0.5  &30  &35  &3600.2  &30  &30  &77.6  &30  &30  &186.5  &3600.0  &4.5  &0  \\
qg.order40  &1600  &62400  &40  &40  &7.5  &40  &40  &1.9  &40  &40  &2.4  &40  &0  &3600.0  &40  &40  &1798.7  &40  &43  &3600.2  &3600.0  &256.4  &8  \\
qg.order60  &3600  &212400  &60  &60  &1813.9  &60  &60  &835.9  &60  &60  &28.8  &60  &0  &3600.1  &60  &62  &3601.5  &60  &0  &3600.0  &3600.0  &315.9  &347  \\
queen10\_10  &100  &1470  &11  &11  &436.4  &11  &11  &758.5  &10  &14  &3600.0  &10  &12  &3600.0  &10  &11  &3600.0  &10  &12  &3600.0  &130.3  &3600.0  &3600.0  \\
queen11\_11  &121  &1980  &11  &11  &404.3  &11  &11  &2705.2  &11  &15  &3600.0  &11  &13  &3600.0  &11  &13  &3600.0  &11  &13  &3600.0  &3602.3  &3600.0  &3600.0  \\
queen12\_12  &144  &2596  &12  &16  &3600.0  &12  &16  &3600.0  &12  &16  &3600.0  &12  &15  &3600.0  &12  &14  &3600.0  &12  &14  &3600.0  &3604.7  &3600.0  &3600.0  \\
queen13\_13  &169  &3328  &13  &17  &3600.0  &13  &17  &3600.0  &13  &17  &3600.0  &13  &16  &3600.0  &13  &16  &3600.0  &13  &15  &3600.0  &3602.3  &3600.0  &3600.0  \\
queen14\_14  &196  &4186  &14  &19  &3600.0  &14  &19  &3600.0  &14  &19  &3600.0  &14  &17  &3600.0  &14  &17  &3600.0  &14  &17  &3600.0  &3602.4  &3600.0  &3600.0  \\
queen15\_15  &225  &5180  &15  &21  &3600.0  &15  &21  &3600.0  &15  &21  &3600.0  &15  &19  &3600.0  &15  &18  &3600.0  &15  &18  &3600.0  &3605.0  &3600.0  &3600.0  \\
queen16\_16  &256  &6320  &16  &23  &3600.0  &16  &23  &3600.0  &16  &23  &3600.0  &16  &21  &3600.0  &16  &20  &3600.0  &16  &19  &3600.0  &3621.9  &3600.0  &3600.0  \\
queen5\_5  &25  &160  &5  &5  &0.0  &5  &5  &0.0  &5  &5  &0.1  &5  &5  &0.0  &5  &5  &0.0  &5  &5  &0.0  &0.0  &0.4  &0  \\
queen6\_6  &36  &290  &7  &7  &0.1  &7  &7  &0.1  &7  &7  &0.0  &7  &7  &0.4  &7  &7  &0.2  &7  &7  &0.1  &0.2  &0.1  &0  \\
queen7\_7  &49  &476  &7  &7  &0.0  &7  &7  &0.0  &7  &7  &0.0  &7  &7  &0.3  &7  &7  &0.2  &7  &7  &0.3  &0.5  &0.1  &0  \\
queen8\_12  &96  &1368  &12  &12  &0.0  &12  &12  &0.1  &12  &12  &0.1  &12  &12  &4.5  &12  &12  &0.4  &12  &12  &0.4  &6.1  &0.3  &0  \\
queen8\_8  &64  &728  &9  &9  &3.4  &9  &9  &2.8  &9  &9  &2.6  &9  &9  &1199.1  &9  &9  &132.8  &9  &9  &31.6  &4.4  &13.9  &1  \\
queen9\_9  &81  &1056  &10  &10  &13.8  &10  &10  &159.2  &10  &10  &921.9  &9  &11  &3600.0  &10  &10  &251.2  &10  &10  &789.0  &7.1  &3600.0  &21  \\
r1000.1  &1000  &14378  &20  &20  &0.4  &20  &20  &0.4  &20  &20  &0.4  &20  &20  &44.2  &20  &20  &7.9  &20  &20  &5.6  &0.8  &0.6  &0  \\
r1000.1c  &1000  &485090  &84  &105  &3600.0  &83  &105  &3600.0  &83  &105  &3600.0  &76  &0  &3601.0  &77  &0  &3600.0  &75  &0  &3601.0  &3600.0  &3600.0  &3600.0  \\
r1000.5  &1000  &238267  &234  &244  &3600.0  &234  &244  &3600.0  &234  &234  &723.1  &212  &0  &3600.1  &212  &0  &3600.0  &212  &0  &3600.0  &3600.0  &- &3600.0  \\
R100\_1g  &98  &503  &5  &5  &0.2  &5  &5  &0.2  &5  &5  &0.3  &5  &5  &0.3  &5  &5  &1.8  &5  &5  &0.7  &384.3  &0.1  &-  \\
R100\_5g  &100  &2456  &12  &18  &3600.0  &12  &18  &3600.0  &12  &18  &3600.0  &10  &17  &3600.0  &12  &16  &3600.0  &12  &16  &3600.0  &2093.7  &3600.0  &-  \\
R100\_9g  &100  &4438  &35  &35  &511.0  &35  &35  &64.2  &35  &35  &63.6  &32  &36  &3600.1  &35  &35  &33.8  &35  &35  &28.5  &3.5  &254.2  &-  \\
r125.1  &122  &209  &5  &5  &0.0  &5  &5  &0.0  &5  &5  &0.0  &5  &5  &0.0  &5  &5  &0.0  &5  &5  &0.0  &0.0  &0.0  &0  \\
r125.1c  &125  &7501  &46  &46  &0.1  &46  &46  &0.0  &46  &46  &0.0  &46  &46  &0.1  &46  &46  &0.0  &46  &46  &0.0  &0.0  &0.1  &0  \\
r125.5  &125  &3838  &36  &36  &0.2  &36  &36  &0.0  &36  &36  &0.2  &33  &36  &3600.0  &36  &36  &0.6  &36  &36  &1.8  &11.6  &0.1  &0  \\
r250.1  &250  &867  &8  &8  &0.0  &8  &8  &0.0  &8  &8  &0.0  &8  &8  &0.0  &8  &8  &0.0  &8  &8  &0.0  &0.0  &0.1  &0  \\
r250.1c  &250  &30227  &64  &64  &0.5  &64  &64  &0.5  &64  &64  &0.5  &64  &64  &0.4  &64  &64  &0.2  &64  &64  &0.2  &36.3  &0.5  &3  \\
r250.5  &250  &14849  &65  &65  &0.5  &65  &65  &0.4  &65  &65  &0.6  &65  &67  &3600.3  &65  &65  &2.9  &65  &65  &2.8  &175.7  &2.2  &2  \\
R50\_1g  &41  &92  &3  &3  &0.0  &3  &3  &0.0  &3  &3  &0.0  &3  &3  &0.0  &3  &3  &0.0  &3  &3  &0.0  &0.5  &0.2  &-  \\
R50\_5g  &50  &612  &10  &10  &0.2  &10  &10  &0.2  &10  &10  &0.1  &10  &10  &7.1  &10  &10  &8.5  &10  &10  &2.5  &0.6  &0.1  &-  \\
R50\_9g  &50  &1092  &21  &21  &0.1  &21  &21  &0.0  &21  &21  &0.1  &21  &21  &0.1  &21  &21  &0.0  &21  &21  &0.0  &0.2  &0.1  &-  \\
R75\_1g  &69  &249  &4  &4  &0.1  &4  &4  &0.0  &4  &4  &0.2  &4  &4  &0.0  &4  &4  &0.0  &4  &4  &0.0  &1.8  &0.1  &-  \\
R75\_5g  &75  &1407  &12  &12  &41.2  &12  &12  &215.8  &12  &12  &1122.1  &10  &13  &3600.0  &11  &13  &3600.0  &11  &13  &3600.0  &215.1  &2029.3  &-  \\
R75\_9g  &75  &2513  &31  &36  &3600.0  &31  &36  &3600.0  &31  &36  &3600.0  &30  &33  &3600.0  &33  &33  &194.2  &33  &33  &135.1  &0.6  &3070.7  &-  \\
school1  &385  &19095  &14  &14  &0.2  &14  &14  &0.1  &14  &14  &0.2  &14  &14  &8.4  &14  &14  &6.5  &14  &14  &7.3  &3623.8  &0.7  &0  \\
school1\_nsh  &352  &14612  &14  &14  &0.2  &14  &14  &0.1  &14  &14  &0.1  &14  &14  &15.5  &14  &14  &12.6  &14  &14  &31.9  &1911.7  &0.3  &0  \\
wap01a  &2368  &110871  &41  &41  &2013.0  &41  &41  &2568.7  &41  &47  &3600.0  &40  &0  &3600.0  &40  &0  &3600.0  &40  &0  &3600.0  &3600.0  &589.8  &3600.0  \\
wap02a  &2464  &111742  &40  &40  &1097.7  &40  &40  &2879.0  &40  &47  &3600.0  &40  &0  &3600.1  &40  &0  &3600.0  &40  &0  &3600.0  &3600.0  &316.4  &3600.0  \\
wap03a  &4730  &286722  &40  &55  &3600.0  &40  &55  &3600.0  &40  &55  &3600.0  &40  &0  &3605.4  &40  &0  &3600.2  &40  &0  &3600.4  &3600.0  &3600.0  &3600.0  \\
wap04a  &5231  &294902  &40  &49  &3600.0  &40  &49  &3600.0  &40  &49  &3600.0  &40  &0  &3600.0  &40  &0  &3600.0  &40  &0  &3600.1  &3600.0  &-  &3600.0  \\
wap05a  &905  &43081  &50  &50  &0.5  &50  &50  &0.7  &50  &50  &1.1  &40  &50  &3600.2  &50  &50  &570.9  &41  &0  &3600.0  &4.3  &1.5  &1  \\
wap06a  &947  &43571  &40  &40  &64.0  &40  &40  &340.4  &40  &40  &380.2  &40  &49  &3600.9  &40  &0  &3600.0  &40  &47  &3600.1  &3600.0  &40.9  &3600.0  \\
wap07a  &1809  &103368  &40  &45  &3600.0  &40  &45  &3600.0  &40  &45  &3600.0  &40  &0  &3600.0  &40  &0  &3600.0  &40  &0  &3600.0  &3600.0  &3600.0  &3600.0  \\
wap08a  &1870  &104176  &40  &40  &461.2  &40  &40  &457.8  &40  &40  &1202.0  &39  &0  &3600.1  &40  &0  &3600.0  &40  &0  &3600.0  &3600.0  &3600.0  &3600.0  \\
will199GPIA  &701  &6772  &7  &7  &0.1  &7  &7  &0.1  &7  &7  &0.1  &7  &7  &0.5  &7  &7  &0.3  &7  &7  &0.3  &2.2  &0.8  &0  \\
zeroin.i.1  &126  &4100  &49  &49  &0.0  &49  &49  &0.0  &49  &49  &0.0  &49  &49  &0.1  &49  &49  &0.1  &49  &49  &0.1  &0.1  &0.1  &0  \\
zeroin.i.2  &157  &3541  &30  &30  &0.1  &30  &30  &0.0  &30  &30  &0.1  &30  &30  &0.8  &30  &30  &0.1  &30  &30  &0.1  &0.0  &0.1  &0  \\
zeroin.i.3  &157  &3540  &30  &30  &0.1  &30  &30  &0.1  &30  &30  &0.0  &30  &30  &0.8  &30  &30  &0.1  &30  &30  &0.1  &0.0  &0.1  &0  \\

\hline
\#solved           &      &       &    &    &\textbf{98}      &    &    &\textbf{98}      &    &    &95      &    &   &72      &    &    &88      &    &    &86      &73      &\textbf{98}       &84      \\
\hline
\end{tabular}
}
\caption*{Table \ref{tableVcpDimacs} (continued)}
\end{table}

\clearpage
\FloatBarrier

\section{Detailed results of the seven models for the 33 DIMACS instances for the BCP}\label{sec:Appendix-BCP}
\begin{table}[h!]
\small
\scalebox{0.64}{
  \setlength{\tabcolsep}{1.5pt} 
  \setlength{\extrarowheight}{1.2pt}
  \begin{tabular}{l  rr || ccc | ccc | ccc | ccc | ccc  | ccc | ccc }
               &    &     &    &\bcppopsat &         &   &\bcppophsat   &         &    &\bcpasssat   &         &    &\bcppopilp  &         &        &\bcppophilp                    &               &        &\bcpassilp                        &         &         &DFMM   &       \\
  Instance     &V   &E    &lb  &ub         &time[s]  &lb  &ub          &time[s]  &lb      &ub     &time[s]  &lb      &ub     &time[s]  &lb      &ub                      &time[s]        &lb      &ub                       &time[s]  &lb       &ub                      &time[s]       \\
    \hline
GEOM20    &20   &20    &21  &21   &0.0     &21  &21   &0.0     &21  &21   &0.0     &21  &21   &0.1     &21  &21  &0.1     &21  &21   &0.1     &21         &21  &0.0       \\
GEOM20a   &20   &37    &20  &20   &0.0     &20  &20   &0.0     &20  &20   &0.0     &20  &20   &0.5     &20  &20  &0.5     &20  &20   &0.6     &20         &20  &0.0       \\
GEOM20b   &20   &32    &13  &13   &0.0     &13  &13   &0.0     &13  &13   &0.0     &13  &13   &0.1     &13  &13  &0.0     &13  &13   &0.1     &13         &13  &0.0       \\
GEOM30    &30   &50    &28  &28   &0.0     &28  &28   &0.0     &28  &28   &0.1     &28  &28   &4.8     &28  &28  &2.6     &28  &28   &0.5     &28         &28  &0.1       \\
GEOM30a   &30   &81    &27  &27   &0.0     &27  &27   &0.0     &27  &27   &0.1     &27  &27   &3.1     &27  &27  &3.2     &27  &27   &2.8     &27         &27  &0.1       \\
GEOM30b   &30   &81    &26  &26   &0.0     &26  &26   &0.0     &26  &26   &0.0     &26  &26   &1.0     &26  &26  &2.7     &26  &26   &0.4     &26         &26  &0.0       \\
GEOM40    &40   &78    &28  &28   &0.0     &28  &28   &0.1     &28  &28   &0.1     &28  &28   &3.6     &28  &28  &8.4     &28  &28   &0.8     &28         &28  &0.1       \\
GEOM40a   &40   &146   &37  &37   &0.2     &37  &37   &0.3     &37  &37   &2.1     &37  &37   &133.1   &37  &37  &17.1    &37  &37   &12.8    &37         &37  &1.4       \\
GEOM40b   &40   &157   &33  &33   &0.1     &33  &33   &0.1     &33  &33   &1.5     &33  &33   &9.3     &33  &33  &240.4   &33  &33   &13.4    &33         &33  &2.1       \\
GEOM50    &50   &127   &28  &28   &0.1     &28  &28   &0.1     &28  &28   &0.2     &28  &28   &6.6     &28  &28  &36.7    &28  &28   &2.3     &28         &28  &0.3       \\
GEOM50a   &50   &238   &50  &50   &0.8     &50  &50   &1.0     &50  &50   &87.1    &50  &50   &280.4   &38  &50  &3600.0  &50  &50   &60.8    &50         &50  &374.4     \\
GEOM50b   &50   &249   &35  &35   &0.5     &35  &35   &0.5     &35  &35   &4.9     &35  &35   &2028.1  &35  &35  &683.9   &35  &35   &250.6   &35         &35  &144.7     \\
GEOM60    &60   &185   &33  &33   &0.2     &33  &33   &0.1     &33  &33   &0.3     &33  &33   &30.2    &33  &33  &56.7    &33  &33   &3.5     &33         &33  &1.1       \\
GEOM60a   &60   &339   &50  &50   &1.0     &50  &50   &0.8     &50  &50   &112.4   &50  &50   &1124.3  &38  &50  &3600.0  &50  &50   &170.4   &50         &50  &684.6     \\
GEOM60b   &60   &366   &41  &41   &2.4     &41  &41   &1.7     &41  &41   &29.3    &34  &41   &3600.0  &33  &41  &3600.0  &40  &42   &3600.0  &41         &41  &22915.9   \\
GEOM70    &70   &267   &38  &38   &0.1     &38  &38   &0.1     &38  &38   &2.3     &38  &38   &36.0    &38  &38  &36.6    &38  &38   &17.6    &38         &38  &2.4       \\
GEOM70a   &70   &459   &61  &61   &6.3     &61  &61   &6.8     &61  &61   &561.0   &44  &61   &3600.0  &35  &62  &3600.0  &60  &61   &3600.0  &61         &61  &24798.0   \\
GEOM70b   &70   &488   &47  &47   &4.9     &47  &47   &5.4     &47  &47   &146.0   &34  &49   &3600.0  &47  &47  &3318.3  &35  &50   &3600.1  &47         &47  &534.6     \\
GEOM80    &80   &349   &41  &41   &0.3     &41  &41   &0.3     &41  &41   &4.0     &41  &41   &132.9   &41  &41  &1103.3  &41  &41   &46.4    &41         &41  &8.2       \\
GEOM80a   &80   &612   &63  &63   &11.4    &63  &63   &10.3    &63  &63   &964.2   &40  &63   &3600.0  &31  &64  &3600.0  &49  &65   &3600.0  &63         &63  &87770.8   \\
GEOM80b   &80   &663   &60  &60   &9.8     &60  &60   &8.1     &60  &60   &894.4   &29  &62   &3600.0  &26  &62  &3600.0  &44  &66   &3600.0  &60         &60  &54320.9   \\
GEOM90    &90   &441   &46  &46   &2.0     &46  &46   &2.0     &46  &46   &21.5    &46  &46   &1679.2  &46  &46  &1607.1  &46  &46   &70.6    &46         &46  &55.2      \\
GEOM90a   &90   &789   &63  &63   &11.0    &63  &63   &12.6    &63  &63   &1447.3  &29  &67   &3600.0  &32  &65  &3600.0  &47  &69   &3600.1  &63         &63  &130050.1  \\
GEOM90b   &90   &860   &69  &69   &58.0    &69  &69   &56.2    &15  &69  &3600.0  &27  &75   &3600.1  &31  &73  &3600.0  &48  &85   &3600.1  &$-\infty$  &69  &172800.0  \\
GEOM100   &100  &547   &50  &50   &1.1     &50  &50   &1.5     &50  &50   &87.8    &33  &50   &3600.0  &41  &50  &3600.0  &50  &50   &238.1   &50         &50  &545.8     \\
GEOM100a  &100  &992   &66  &66   &185.0   &66  &66   &316.3   &13  &67   &3600.0  &27  &75   &3600.0  &36  &72  &3600.0  &35  &81   &3600.0  &$-\infty$  &70  &172800.0  \\
GEOM100b  &100  &1050  &71  &71   &156.6   &71  &71   &136.3   &15  &71  &3600.0  &24  &78   &3600.0  &22  &78  &3600.0  &42  &86   &3600.1  &$-\infty$  &71  &172800.0  \\
GEOM110   &110  &638   &50  &50   &1.7     &50  &50   &1.7     &50  &50   &98.9    &37  &50   &3600.0  &40  &51  &3600.0  &50  &50   &319.1   &50         &50  &2982.2    \\
GEOM110a  &110  &1207  &69  &69   &287.1   &69  &69   &160.5   &14  &70  &3600.0  &26  &77   &3600.0  &25  &79  &3600.0  &43  &88   &3600.1  &$-\infty$  &73  &172800.0  \\
GEOM110b  &110  &1256  &77  &77   &620.8   &77  &77   &608.1   &15  &78  &3600.0  &21  &86   &3600.0  &22  &84  &3600.0  &40  &95   &3600.1  &$-\infty$  &79  &172800.0  \\
GEOM120   &120  &773   &59  &59   &4.2     &59  &59   &4.5     &59  &59   &516.6   &24  &61   &3600.0  &43  &59  &3600.0  &59  &59   &609.2   &59         &59  &10778.2   \\
GEOM120a  &120  &1434  &82  &82   &338.2   &82  &82   &293.0   &16  &84  &3600.0  &19  &90   &3600.0  &28  &93  &3600.0  &46  &92   &3600.1  &$-\infty$  &84  &172800.0  \\
GEOM120b  &120  &1491  &16  &83  &3600.0  &16  &83  &3600.0  &16  &85  &3600.0  &19  &102  &3600.0  &28  &97  &3600.0  &39  &100  &3600.1  &$-\infty$  &85  &172800.0  \\
\hline
\#solved  &    &     &   &    &\textbf{32}      &   &    &\textbf{32}      &   &    &26      &   &    &17      &   &   &16      &   &    &20      &          &   &26        \\
\hline
\end{tabular}
}
\caption{Results of the seven models for the 33 DIMACS instances for the BCP}
\label{tableBcpDimacs}
\end{table}

\end{document}